\documentclass{article}

\usepackage{microtype}
\usepackage{graphicx}
\usepackage{subcaption}
\usepackage{booktabs} % for professional tables
\usepackage{thmtools}

\usepackage{tikz}
\usepackage[hyphens]{url}
\usetikzlibrary{arrows.meta, positioning, matrix, decorations.pathreplacing}

\usepackage{hyperref}

\usepackage[preprint]{icml2026}

\usepackage{amssymb}
\usepackage{mathtools}
\usepackage{amsmath}
\usepackage{amsthm}

\usepackage{xcolor}
\usepackage{fancyvrb}
\usepackage{fvextra}
\usepackage{aliascnt}
\usepackage{textcomp}
\usepackage{enumitem}
\DeclareRobustCommand{\visiblespace}{%
  \ifmmode\text{\textvisiblespace}\else\textvisiblespace\fi
}

\usepackage[capitalize,noabbrev]{cleveref}
\theoremstyle{plain}
\newtheorem{theorem}{Theorem}[section]

\newaliascnt{proposition}{theorem}
\newtheorem{proposition}[proposition]{Proposition}
\aliascntresetthe{proposition}

\newaliascnt{lemma}{theorem}
\newtheorem{lemma}[lemma]{Lemma}
\aliascntresetthe{lemma}

\newaliascnt{corollary}{theorem}

\aliascntresetthe{corollary}

\newaliascnt{informaldef}{theorem}

\aliascntresetthe{informaldef}

\theoremstyle{definition}
\newaliascnt{definition}{theorem}
\newtheorem{definition}[definition]{Definition}
\aliascntresetthe{definition}

\theoremstyle{definition}
\newaliascnt{example}{theorem}
\newtheorem{example}[example]{Example}
\aliascntresetthe{example}

\newaliascnt{assumption}{theorem}

\aliascntresetthe{assumption}

\theoremstyle{remark}
\newaliascnt{remark}{theorem}

\aliascntresetthe{remark}

\crefname{section}{section}{sections}
\Crefname{section}{Section}{Sections}

\crefname{appsec}{appendix}{appendices}
\Crefname{appsec}{Appendix}{Appendices}

\makeatletter
\newcommand{\applabel}[1]{%
  \crefalias{section}{appsec}%
    \crefalias{subsection}{appsec}%
  \label{#1}%
  \crefalias{section}{section}%
  \crefalias{subsection}{subsection}%

}
\makeatother

\usepackage[textsize=tiny]{todonotes}

\icmltitlerunning{BAPO Bounds on Chain-of-Thought Token Complexity in LLMs}

\newenvironment{prompt}
  {\VerbatimEnvironment
   \begin{Verbatim}[fontsize=\small, formatcom=\color{black},
                    frame=single, rulecolor=\color{black},
                    breaklines=true,          % wrap long lines
                    breakanywhere=true,       % allow breaking long tokens/URLs
                         breaksymbol={},           % no symbol at line breaks
                     breaksymbolleft={},       % ensure no left-side symbol
                     breaksymbolindent=0pt,
                    framerule=0.3pt, xleftmargin=1em,
                    commandchars=\\\{\}]}
  {\end{Verbatim}}

\usepackage{xspace}
\newcommand{\bopa}{cBAPO\xspace}
\newcommand{\bopas}{cBAPOs\xspace}
\newcommand{\bopabapos}{BAPOs/cBAPOs\xspace}

\begin{document}

\twocolumn[
  \icmltitle{Reasoning About Reasoning: \\ BAPO Bounds on Chain-of-Thought Token Complexity in LLMs}
  


  \icmlsetsymbol{equal}{*}

  \begin{icmlauthorlist}
    \icmlauthor{Kiran Tomlinson}{msr}
    \icmlauthor{Tobias Schnabel}{msr}
    \icmlauthor{Adith Swaminathan}{netflix}
    \icmlauthor{Jennifer Neville}{msr}

  \end{icmlauthorlist}

  \icmlaffiliation{msr}{Microsoft Research, Redmond, WA}
  \icmlaffiliation{netflix}{Netflix, Los Gatos, CA}

  \icmlcorrespondingauthor{Kiran Tomlinson}{kitomlinson@microsoft.com}

  \icmlkeywords{chain of thought, lower bounds, bounded attention, transformers}

  \vskip 0.3in
]



\printAffiliationsAndNotice{}  % no special notice (required even if empty)

\begin{abstract}
Inference-time scaling via chain-of-thought (CoT) reasoning is a major driver of state-of-the-art LLM performance, but it comes with substantial latency and compute costs. 
We address a fundamental theoretical question: \emph{how many} reasoning tokens are required to solve a problem as input size grows?  
By extending the bounded attention prefix oracle (BAPO) model---an abstraction of LLMs that quantifies the information flow required to solve a task---we prove lower bounds on the CoT tokens required for three canonical BAPO-hard tasks: binary majority, triplet matching, and graph reachability. 
We show that each requires $\Omega(n)$ reasoning tokens when the input size is $n$. 
We complement these results with matching or near-matching upper bounds via explicit constructions. 
Finally, our experiments with frontier reasoning models show approximately linear reasoning token scaling on these tasks and failures when constrained to smaller reasoning budgets, consistent with our theoretical lower bounds. 
Together, our results identify fundamental bottlenecks in inference-time compute through CoT and offer a principled tool for analyzing optimal reasoning length.

\end{abstract}

\section{Introduction}
\label{sec:intro}

\begin{figure*}[tbp]
  \centering
  \begin{subfigure}[t]{0.48\linewidth}
    \centering
    \resizebox{\linewidth}{!}{
      \begin{tikzpicture}
        \matrix (m) [
          matrix of nodes,
          nodes in empty cells,ampersand replacement=\&,
          row sep=0pt,
          column sep=1mm,
          nodes={
            draw,
            line width=1pt,
            minimum height=1cm,
            minimum width=0.9cm,
            align=center,
            anchor=center,
            font=\normalsize
          },
          column 1/.style={nodes={minimum width=5cm}, column sep=0.3cm},
          column 5/.style={nodes={draw=none}},
          column 6/.style={column sep=0.5cm}
        ]{
          $x_1 \qquad \dots \qquad x_n$ \& $s_1$ \& $s_2$ \& $s_3$ \& $\ldots$ \& $s_m$ \& $y$ \\
        };

        \draw [decorate, decoration={brace, amplitude=5pt}]
    ([yshift=6pt]m-1-1.north west) -- ([yshift=6pt]m-1-1.north east)
    node[midway, yshift=13pt] {input};

  \draw [decorate, decoration={brace, amplitude=5pt}]
    ([yshift=6pt]m-1-2.north west) -- ([yshift=6pt]m-1-6.north east)
    node[midway, yshift=13pt] {CoT reasoning};

  \draw [decorate, decoration={brace, amplitude=5pt}]
    ([yshift=6pt]m-1-7.north west) -- ([yshift=6pt]m-1-7.north east)
    node[midway, yshift=13pt] {output};

   \draw [decorate, decoration={brace, amplitude=5pt}]
        ([yshift=-6pt]m-1-2.south east) -- ([yshift=-6pt]m-1-1.south west)
        node(step)[midway, yshift=-13pt] {BAPO step 1};
        
   \draw [decorate, decoration={brace, amplitude=5pt}]
        ([yshift=-26pt]m-1-3.south east) -- ([yshift=-26pt]m-1-1.south west)
        node(step)[midway, yshift=-13pt] {BAPO step 2};
    \draw [decorate, decoration={brace, amplitude=5pt}]
        ([yshift=-46pt]m-1-4.south east) -- ([yshift=-46pt]m-1-1.south west)
        node(step)[midway, yshift=-13pt] {$\ldots$};

      \end{tikzpicture}
    }
    \caption{CoT tokens produced by an LLM are modeled through individual BAPO steps for each produced CoT token (\Cref{def:bapo-cot}).}
    \label{fig:bapo}
  \end{subfigure}
  \hfill
  \begin{subfigure}[t]{0.48\linewidth}
    \centering
    \resizebox{\linewidth}{!}{%
      \begin{tikzpicture}[>=Latex, thick, font=\small]
        % Main block
        \draw (0,0) rectangle (8,1);
        \node at (2,0.5) { $x_1\qquad \ldots \qquad x_k$};
        \node at (6,0.5) {$x_{k+1} \ldots x_n \;s_1 \ldots s_{t-1}$};

        % Split line
        \draw[very thick] (4,-0.4) -- (4,1.3);
        \node[above] at (4,1.2) {$\forall$ splits $k$};

        % y box
        \draw (8.6,0) rectangle (9.6,1);
        \node at (9.1,0.5) {$s_t$};

        % Prefix/attention arrows
        \draw[->] (1.2,-0.1) -- (1.2,-0.55)
          -- node[midway, fill=white, inner sep=3pt] {prefix $f(\cdot)$: $a$ bits} (9.1,-0.55)
          -- (9.1,0);
        \draw[->] (1.2,-0.1) -- (1.2,-1.15)
          -- node[midway, fill=white, inner sep=3pt] {attention $g(\cdot)$: $b$ tokens} (9.1,-1.15)
          -- (9.1,0);
      \end{tikzpicture}%
    }
    \caption{Every step is BAPO-easy; i.e., constant communication needed across an arbitrary split of the input and reasoning tokens.}
    \label{fig:computegraph}
  \end{subfigure}
  \hfill
    \caption{Overview of our modeling approach. An LLM autoregressively generates reasoning tokens $s_1, \cdots s_m$ to eventually produce an answer $y$. Each token generation step is modeled using a bounded attention prefix oracle (BAPO) with constant bandwidths; see~\Cref{sec:multistep_bapo} for a formal definition. To produce each $s_t$, information in a BAPO can cross an arbitrary prefix/suffix cut only through an $a$-bit summary and up to $b$ attended tokens from the prefix; this constrained update repeats for steps $t=1, \ldots m+1$.}
    \label{fig:schematic}
\end{figure*}
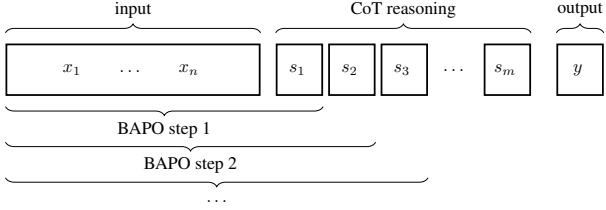
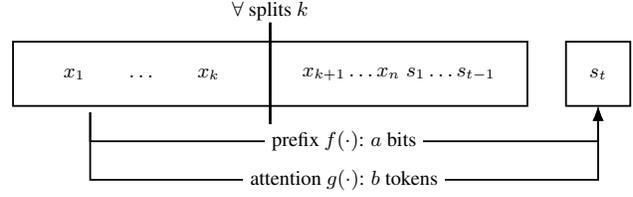

State-of-the-art transformer-based large language models (LLMs) rely on inference-time scaling---allocating additional test-time compute by generating intermediate tokens---to improve performance on challenging tasks~\cite{snell2024scaling}. 
A main drawback of inference-time scaling is that it is costly: even with KV caching, generating $n$ reasoning tokens entails superlinear compute and latency due to the cost of attention growing quadratically~\cite{vaswani2017attention}. 

This raises the following open question: how does the need for reasoning tokens scale with input size for a given task? In particular, what is the minimum reasoning budget required to reliably solve a task? 

There is a growing body of empirical work on compressing and optimizing reasoning chains~\cite{sui2025stopoverthinkingsurveyefficient}. 
However, these evaluations are conducted on benchmarks where input size is not a natural scaling parameter, and compression is therefore reported as a constant-factor reduction (e.g. ``$30\%$ shorter''). As input sizes grow, does the required reasoning length remain roughly constant, grow linearly, or even faster? Principled guidance and theoretical understanding of when compression is even possible is still largely missing. This paper helps fill in that gap by providing asymptotic lower and upper bounds on reasoning chain length for three problems that are hard for LLMs to solve without reasoning~\cite{schnabel2025lost}. The resulting lower bounds and experimental results help confirm the \emph{token complexity hypothesis} of \citet{lee2025llmscompresscot}, which states that every problem has an inherent minimal reasoning chain length, and align with reasoning length bounds for unique hard-attention transformers~\cite{amiri2025lower}.

Our theory relies on a previously identified bottleneck in transformers, namely that the amount of information a transformer can accurately process while producing a single token is limited. 
This is formalized in the \emph{bounded attention prefix oracle} (BAPO) model~\citep{schnabel2025lost}. For problems that are hard in this framework, i.e., require a lot of information to be communicated, chain of thought~\citep[CoT;][]{wei2022cot}\footnote{We use CoT in a more general sense to denote all token-based reasoning approaches, following community use.} enables LLMs to solve them. We show in this paper that these BAPO-hard problems carry high CoT inference costs both in theory and practice.

In particular, we prove that three BAPO-hard problems---binary majority, triplet matching, and graph reachability---require at least $\Omega(n)$ CoT tokens to be solved by a constant-bandwidth BAPO with input size $n$. Consequently, for these tasks, we expect the cost and latency of CoT to scale at least as fast as the problem size, regardless of the particular CoT strategy employed.~\cref{fig:schematic} shows an information flow schematic of chain of thought decoding of LLMs and our multi-step BAPO abstraction of that process. We complement these lower bounds with matching or nearly-matching upper bounds on the token complexity of the three problems through explicit constructions. Additionally, we prove that there are BAPO-hard problems with sublinear token complexity by showing that pointer chasing is BAPO-hard, but has token complexity scaling with the length of the pointer chain rather than the input. 

Finally, we test frontier reasoning models to see how many reasoning tokens they use in practice on these problems. We find that GPT-5.2 with reasoning exhibits linear token scaling when solving the binary majority, triplet matching, and graph reachability problems, as predicted by our theory. Gemini 2.5 Pro uses even more tokens. We also find that imposing a too small reasoning budget degrades performance substantially,  consistent with our lower bounds.

Overall, our results show that inference-time scaling via reasoning tokens can be a bottleneck  on hard tasks: linear lower bounds on required reasoning length translate into at least quadratic compute in practice as context grows, motivating the need for alternatives such as new architectures or more effective tool use. 
At the same time, our empirical results show that frontier reasoning models often use substantially more tokens than the linear lower-bounds require, suggesting room for constant-factor efficiency gains even when sublinear scaling is information-theoretically impossible under BAPO bandwidth limits. 

\section{Background}
\Cref{fig:bapo} illustrates how an LLM generates chain-of-thought tokens by repeatedly applying the same architecture to an ever-growing context.  Although powerful, LLMs are limited in what they can compute in a single token-generation step. We rely on the recently introduced BAPO model~\cite{schnabel2025lost} that postulates that LLMs are communication-bound: if a problem requires too much information to be communicated inside of a transformer, the computation fails. The amount of communication  a problem requires is measured by two parameters of the BAPO model, the \emph{prefix bandwidth} $a$ and the \emph{attention bandwidth} $b$, illustrated in \Cref{fig:computegraph}, that capture different communication pathways in transformers.

Intuitively, the final layer of a transformer must have all relevant information about the input in order to produce a correct next-token prediction. The BAPO model captures this by considering how input information crosses an arbitrary line dividing the input into a prefix and a suffix. Preprocessed information about the prefix (i.e., representations after each layer) can cross this line, but due to causal attention, the preprocessing done in the prefix must be independent of the suffix. Limits on this information flow pathway are captured by the prefix bandwidth. Alternatively, attention in the suffix can pull in prefix token values in a way that depends on the suffix, and the attention bandwidth captures how many tokens can be effectively attended to without information loss. We provide a deeper introduction to the BAPO model in \Cref{app:bapo-intro} and refer readers to \citet{schnabel2025lost} for a thorough exposition.

\subsection{Single-step BAPO}
Let $p:\Sigma^* \rightarrow \Sigma$ be the computational problem we would like to solve over a token alphabet $\Sigma$.

\begin{definition}
An \emph{$(a, b)$-(score-)BAPO}\footnote{We use the slightly more general \emph{score-BAPO} from \citet{schnabel2025lost} but we refer to it as BAPO for simplicity.} with \emph{prefix bandwidth} $a$ and \emph{attention bandwidth} $b$ is defined by 
\begin{itemize}[itemsep=1pt,topsep=-5pt,parsep=0pt,partopsep=0pt]
    \item a \emph{prefix oracle} $f:\Sigma^* \rightarrow \{0, 1\}^a$,
    \item an \emph{attention function} $g: \Sigma^* \times \mathbb N  \times \Sigma \times \mathbb N\rightarrow [0, 1]$
    \item and a \emph{suffix oracle} $h: \{0, 1\}^a \times (\cup_{i = 0}^b (\Sigma \times \mathbb N)^i) \times \Sigma^* \times \mathbb N \rightarrow \Sigma$.
\end{itemize}
An $(a, b)$-BAPO \emph{solves} $p$ if $$h(f(x_1 \dots x_k), G, x_{k+1}\dots x_n, k) = p(x_1\dots x_n)$$ 
for all $k < n$, and all $$G \in \operatorname*{arg\,max}_{\substack{S \subseteq \{(x_i, i): i \le k, g_i > 0\} \\ |S| \le b}} \sum_{(x_i, i) \in S} g_i,$$ where $g_i = g(x_{k+1}\dots x_n,k, x_i, i)$. (See \Cref{app:bapo-intro} for a walk-through of this definition.)
\end{definition}

\begin{definition}
    A problem $p$ is \emph{BAPO-easy} if there exists a constant-bandwidth $(O(1), O(1))$-BAPO that solves $p$. Otherwise, the problem is called \emph{BAPO-hard}. 
\end{definition}

\begin{table*}[t!]
\centering
    \caption{Summary of token complexity results. }
    \label{tab:summary}
    \begin{tabular}{lrrr}
    \toprule
        \bfseries Problem& \bfseries Lower bounds& \multicolumn{2}{c}{\bfseries Upper bounds} \\ \cmidrule(lr){2-2} \cmidrule(lr){3-4}
 & \bopa/BAPO& \bopa&BAPO\\
         \midrule
         \textsc{Majority} & $\Omega(n)$ (Thm.~\ref{thm:majority-lb}) & $O(n \log n)$ (Thm.~\ref{thm:majority-ub}) & $n+1$ (Prop.~\ref{prop:double-trouble-cot})\\
          \textsc{Match3}$_n$ & $\Omega(n)$ (Thm.~\ref{thm:match3-lb}) & $O(n)$ (Thm.~\ref{thm:match3-ub}) & $n+1$ (Prop.~\ref{prop:double-trouble-cot})\\
          \textsc{Reachability} & $\Omega(m)$ (Thm.~\ref{thm:reachability-lb}) & $O(n^2)$ (Thm.~\ref{thm:reachability-ub}) & $m+3$ (Prop.~\ref{prop:double-trouble-cot})\\
          \midrule
          Any decidable language & & $O(t(n)\max(n, s(n)))^\dagger$ (Thm.~\ref{thm:tm-ub}) & $n+1$ (Prop.~\ref{prop:double-trouble-cot})\\
          Any language & & (open) & $n+1$ (Prop.~\ref{prop:double-trouble-cot})\\
          Any \bopa-easy problem & & 1 (by definition) & 1 (by definition)\\
         \bottomrule
    \end{tabular}
    
    \footnotesize For \textsc{Reachability}, $n$ and $m$ are the number of  nodes and  edges; in all other cases, $n$ is the input length.
    
    $^\dagger$$t(n)$, $s(n)$ are the time/space complexity of a TM for the problem using any constant number of tapes and heads.
\end{table*}

\subsection{BAPO-CoT}
\label{sec:multistep_bapo}
A BAPO with chain of thought is an extension of the original model where each reasoning token is produced by a BAPO. As shown by~\citet{schnabel2025lost}, any decidable problem $p$ (even BAPO-hard ones) can be broken up into a sequence of BAPO-easy instances of a step-by-step problem $p'$. This means that CoT can lower required bandwidth to a constant given sufficient reasoning tokens.

\begin{definition}
    \label{def:bapo-cot}
    An \emph{$(a, b)$-BAPO-CoT} solving $p:\Sigma^* \rightarrow \Sigma$ is an $(a, b)$-BAPO over the token set $\Gamma \supseteq \Sigma \cup  \{\square\}$ 
    that solves some problem $p':\Gamma^* \rightarrow \Gamma$ 
    such that for all inputs $x\in \Sigma^*$ to $p$, there exists some sequence of strings $s_1, \dots, s_m \in \Gamma^*$ with the following properties.
    \begin{enumerate}
        \item $s_1 = x$ (the BAPO-CoT starts with $x$ as its input),
        \item $s_{i+1} = s_{i}p'(s_i)$ for all $i = 1, \dots, m-1$ (at each step, it produces some CoT token, concatenated to $s_i$),
        \item $p'(s_{i}) = \square$ if and only if $i = m-1$ (at the last step only, it outputs the halt token), and
        \item $p'(s_{m-2}) = p(x)$  (solves the problem before halting).
    \end{enumerate}
\end{definition}

\subsection{Problems}

We analyze the following problems known to be BAPO-hard~\cite{schnabel2025lost}, thus requiring CoT to solve with constant bandwidth.

\begin{definition}
    \textsc{Majority}: $\{0, 1\}^n \rightarrow \{0, 1\}$ is the problem of determining whether the input has strictly more ones than zeros.
\end{definition}

\begin{definition}
  \textsc{Match3}$_n$: $\mathbb Z_m^n \rightarrow \{0, 1\}$ is the problem of determining whether there are some $i,j\in [n]$ ($i < j < n$) such that $x_{i} + x_{j} + x_n \equiv 0$ (mod $m$), with output 1 if such a match exists and 0 otherwise.
\end{definition}

\begin{definition}
    \textsc{Reachability}: $([n] \times [n])^m \times [n] \times [n] \rightarrow \{0, 1\}$ is the problem of determining if there is path from $s$ to $t$ in a directed graph $G$ with $n$ nodes and $m$ edges. To encode the problem, let $\Sigma = [n]\times [n] \cup [n]$ where the token $(i, j) \in \Sigma$ represents edge $(i, j)$ and the integer tokens represent nodes. An instance of \textsc{Reachability} is specified by the edge list of $G$ in any order followed by nodes $s$ and $t$. 
\end{definition}

BAPO-easy examples include \textsc{Parity}, \textsc{Equality}, \textsc{Disjointness} and \textsc{Index}~\cite{schnabel2025lost}. 

\section{BAPO Token Complexity}
We now introduce a formal notion of CoT efficiency under the BAPO model. Unless needed for clarity, we refer to BAPO token complexity simply as token complexity.

Let $\mathcal P_n$ be the set of all instances of a problem $p$ of size $n$. Let $\mathcal B(\mathcal P_n)$ be the set of all constant-bandwidth BAPO-CoTs that correctly solve all instances in $\mathcal P_n$, with $|\Gamma| = \Theta(|\Sigma|)$.\footnote{If the CoT alphabet $\Gamma$ is arbitrarily large, in particular scaling arbitrarily with $n$, then this allows the CoT tokens to store too much---for instance, storing 1s counts in \textsc{Majority} for arbitrarily long inputs. But for some problems, $|\Sigma|$ varies with problem size (e.g., in \textsc{Reachability}, where $\Sigma$ needs to encode edges uniquely), so we need to allow $\Gamma$ to grow with $\Sigma$.}
\begin{definition}
    \label{def:tokencomplexity}
    The \emph{BAPO token complexity} $c(n)$ of a problem is given by the constant-bandwidth BAPO-CoT with best worst-case performance for $n$:
    \begin{equation*}
        c(n) = \min_{(f, g, h) \in \mathcal B(\mathcal P_n)} \max_{x\in \mathcal P_n} \mathsf{cotTokens}\big((f, g, h), x\big).
    \end{equation*}
\end{definition}

This means that no constant-bandwidth BAPO-CoT can solve all instances of $p$ with fewer than $c(n)$ CoT tokens.    
Here, $\mathsf{cotTokens}((f, g, h), x)$ is the worst-case number of CoT tokens a BAPO-CoT $(f, g, h) \in \mathcal B(\mathcal P_n)$ produces on $x\in \mathcal P_n$ before the halt token $\square$ (worst-case over split indices $k$ and choices of attended token sets $G$). This is a \emph{nonuniform} notion of token complexity, as we allow the optimal BAPO to be different for every input length $n$. 

We will also consider a restricted version of BAPOs, called \emph{self-consistent BAPO} (\bopa) for more meaningful token complexity upper bounds. See \Cref{sec:cbapo} and \Cref{app:bapo*} for more detail.

\subsection{Summary of Results}

Our objective is to characterize the token complexity of problems, which we will do through upper bounds (constructions) and lower bounds (impossibility proofs). \Cref{tab:summary} shows a summary of our results, both for the original BAPO-CoT model as well as the \bopa-CoT model (replacing BAPO with \bopa in \Cref{def:tokencomplexity}).

Our main theoretical contribution is that \textsc{Majority}, \textsc{Match3}$_n$, and \textsc{Reachability} all require linearly many tokens for a constant-bandwidth BAPO-CoT. Note that the \bopa and BAPO models share the same lower bounds. This is because a lower bound for the less restricted BAPO directly implies the same token complexity bound for the restricted model, and our proofs do not make use of the restriction. The table also highlights why we introduce the more restricted \bopa model. Without its restriction that limits the power of the prefix oracle $f$, we can exploit the power of $f$ to solve any problem (even undecidable ones) with linearly many tokens, which is clearly unrealistic.

\section{Lower Bounds on Token Complexity}\label{sec:lower-bounds}
We begin with token complexity lower bounds on three BAPO-hard problems and then discuss how classical algorithmic lower bounds relate to token complexity.

\subsection{BAPO-CoT Token Complexity Lower Bounds}

The key to our analysis is an extension of the BAPO-hardness proof technique of \citet{schnabel2025lost} that accounts for every possible reasoning chain a BAPO-CoT could construct. The core idea behind the original BAPO-hardness proofs is, given a supposed $(a, b)$-BAPO solving the problem, to construct a collection of problem instances split into prefixes and suffixes with the following properties: 
\begin{enumerate}[label=(\arabic*),itemsep=1pt]
    \item the prefixes have a set of \emph{masking tokens} in common, placed strategically so that the BAPO will always attend to masking tokens, 
    \item the set of prefixes is sufficiently large ($> 2^a$) so that the prefix oracle has the same output for two of them by the pigeonhole principle, 
    \item given any pair of prefixes, there exists some suffix where the answer to the problem is different if we append the suffix to the two prefixes. 
\end{enumerate}
 These ingredients provide a scenario where the BAPO sees identical information (same prefix oracle output, same attended tokens, same suffix) in two instances with different answers, meaning it answers one of them wrong. Collections of such prefixes and suffixes do not exist for all problems, and managing their size and properties is the core challenge of the proof.

We perform the same style of construction, but additionally ensure that the information seen by the BAPO-CoT is identical between the two instances \emph{at every CoT step}. This means the two BAPO-CoTs generate the same reasoning sequence and, eventually, the same answer. This requires many more masking tokens, but we show that the three problems can all accommodate them in the constructions. See \Cref{app:proofs} for all proofs.

\begin{theorem}\label{thm:majority-lb}
    Any constant-bandwidth BAPO-CoT for \textsc{Majority} on length-$n$ inputs uses $\Omega(n)$ CoT tokens, so the token complexity of \textsc{Majority} is at least $\Omega(n)$.
\end{theorem}

\begin{theorem}\label{thm:match3-lb}
    Any constant-bandwidth BAPO-CoT for \textsc{Match3}$_n$ over $\mathbb Z _{n^2}^n$ uses $\Omega(n)$ CoT tokens, so the token complexity of \textsc{Match3}$_n$ is at least $\Omega(n)$.
\end{theorem}

\begin{theorem}\label{thm:reachability-lb}
    Any constant-bandwidth BAPO-CoT for \textsc{Reachability} on $m$-edge graphs uses $\Omega(m)$ CoT tokens, so the token complexity of \textsc{Reachability} is at least $\Omega(m)$.
\end{theorem}

\subsection{Lower Bounds via Turing Machines}
Beyond these three problems, Turing machine lower bounds can provide a general, albeit weak approach to lower bound chain-of-thought length. However, this approach provides lower bounds for transformers, but not for BAPOs. 

To turn algorithmic lower bounds into CoT lower bounds, we can simulate a transformer with a Turing machine.  More precisely, if any possible algorithm for a problem requires time $\Omega (t(n))$, then no CoT for that problem can be shorter than $\Omega(t(n) / n^2)$ assuming the forward pass of a transformer takes time $O(n^2)$; otherwise, running the transformer with CoT would give a more efficient algorithm. By the time hierarchy theorem~\cite{hartmanis1965computational}, this roughly means there exist problems with any desired minimum CoT length (e.g., there exist problems with minimum CoT length $\Omega(n^a)$ for all $a$). This also means, for instance, that \textsf{NP}-hard problems do not have polynomial-length CoTs unless  \textsf{P} = \textsf{NP}. But unconditional time complexity lower bounds for specific problems (beyond $\Omega(n)$, which is useless after applying the $1/n^2$ factor) are virtually nonexistent, so this approach is unlikely to provide useful lower bounds in practice for problems in \textsf{P} like graph reachability. 

While this approach provides few useful bounds, we can still ask whether it can be used for BAPO-CoTs. It was previously unknown whether constant-bandwidth BAPOs could be simulated by Turing machines at all. We resolve this question in the negative, so algorithmic lower bounds do not immediately translate to BAPO-CoT token complexity lower bounds. The proof uses a unbounded token alphabet to encode hard problems in a single token, although \Cref{prop:double-trouble} shows this impossibility still holds for BAPOs with finite $\Sigma$ (the question for cBAPOs with finite $\Sigma$ is open).

\begin{proposition}\label{prop:no-tm-simulation}
    Constant-bandwidth \bopabapos 
    cannot in general be simulated by Turing machines.
\end{proposition}
\begin{proof}
     Pick $\Sigma = \mathbb N$ and encode Turing machines as natural numbers. The problem  $p(x_1\dots x_n) =1[\text{$x_n$ is a Turing machine that halts on all inputs}]$ is undecidable yet solved by a $(0, 0)$-\bopa
     since $h$ always has the final token in the suffix and can be an arbitrary function (it's trivially self-consistent since $a=0$). 
\end{proof}

\section{Upper Bounds on Token Complexity}\label{sec:upper-bounds}
We first introduce the restricted \bopa model and discuss why it is needed. We then provide general Turing-machine-based upper bounds on token complexity, followed by tighter bounds for the three BAPO-hard problems via explicit construction, and closing with an example of a problem with sublinear token complexity.

\subsection{Self-Consistency and \bopa}\label{sec:cbapo}

In our analysis of upper bounds, we will focus on a restricted version of the original BAPO model that limits the power of the prefix oracle $f$ in the original model to align more closely with the way transformer computations happen.  There is a certain self-similarity in how causally-masked transformers work: the result of the computation on a sequence $x_1, \ldots, x_{k+1}$, is the result of the transformers' new computation on $x_{k+1}$ merged with its old computation on the prefix $x_1, \ldots, x_k$. This motivates the additional constraint of \emph{self-consistency}: any prefix result needs to be computable by a BAPO. In addition to capturing this computational constraint in real transformers better, it also provides more realistic upper bounds on token complexity.

 \begin{restatable}{definition}{SelfConsistencyDefinition}\label{def:bapo-consistent}
    A function $f: \Sigma^* \rightarrow \{0, 1\}^{a}$ is $(a, b)$-\emph{BAPO-consistent} for a problem $p$ if there exist functions $g$ and $h$ such that $(f, g, h)$ is an $(a, b)$-BAPO with input token set $\Sigma$ and output token set $\{0, 1\}^a$ for which $f(x_1\dots x_n) = h(f(x_1 \dots x_k), G, x_{k+1}\dots x_n, k)$ for all $n> 1$, all prefixes $x_1\dots x_n$ of inputs to $p$, all split indices $k< n$, and all possible attended token sets $G$. 
\end{restatable}

We call BAPOs with this restriction \bopas, and define \bopa-easiness and \bopa-hardness analogously. % as for BAPOs. 
\begin{restatable}{definition}{CBAPODefinition}
    A $(a, b)$-BAPO $(f, g, h)$ where $f$ is $(a, b)$-BAPO-consistent is called \emph{self-consistent}, denoted $(a, b)$-\emph{\bopa}. A problem $p$ is \emph{\bopa-easy} if an $(O(1), O(1))$-\bopa solves $p$. Otherwise, $p$ is \emph{\bopa-hard}.
\end{restatable}

For more details and analysis of \bopa see \Cref{app:bapo*}.
In particular, we show that this new definition preserves all meaningful known properties of the model (i.e., all previously-mentioned BAPO-easy and BAPO-hard problems are cBAPO-easy and cBAPO-hard, respectively).

The motivation for studying \bopas for token complexity upper bounds is that they eliminate a surprising and undesirable property of BAPOs, namely that repeating the input twice allows a constant-bandwidth BAPO to solve any problem. Intuitively, this is because either the prefix oracle $f$ or the suffix oracle $h$ sees a whole copy of the input, regardless of the split index. This behavior is less problematic in the single-token setting as the standard BAPO cannot copy the input, but causes issues in the CoT setting.   

\begin{proposition}[The Input-Doubling Loophole]\label{prop:double-trouble}
Let $L$ be any language over $\Sigma$. Add a special token $| \notin \Sigma$ to the alphabet. Define the problem $p_L : \{x|x : x \in \Sigma^*\} \rightarrow \{0, 1\}$ where $p_L(x|x) = \mathbf1[x \in L]$. There is a $(1, 0)$-BAPO solving $p_L$.
\end{proposition}
\begin{proof}
Given input $x|x$ to the BAPO, $f$ outputs $1[x \in L]$ if $|$ appears in the prefix and 0 otherwise. The suffix oracle $h$ outputs $1[x \in L]$ if $|$ appears in the suffix, or outputs what it gets from $f$ otherwise. This BAPO has prefix bandwidth 1, doesn't use attention, and solves $p_L$.
\end{proof}

This makes upper bounds on token complexity unrealistically low as we can construct a linear-time constant-bandwidth BAPO-CoT algorithm for any problem. The BAPO-CoT only needs to write a second copy of the input, at which point it can use the input-doubling loophole to solve the problem. 

\begin{proposition}\label{prop:double-trouble-cot}
    For any decision problem $p: \Sigma^* \rightarrow \{0, 1\}$ and any $n$, there is a $(1, 1)$-BAPO-CoT that solves $p$ on inputs up to size $n$ with at most $n+1$ tokens.
\end{proposition}

\begin{proof}
    The idea behind the BAPO-CoT is that it only needs to write down a second copy of the input, at which point the problem can be solved in one step by \Cref{prop:double-trouble}.
    
    Let \visiblespace{} be a symbol not in $\Sigma$. Let $m$ be the length of the input to the BAPO-CoT (with $m \le n$ at the first step). First, while $m < n$, $h$ outputs \visiblespace. The attention function attends to the symbol at index $m - n+1$. While $n \le m < 2n$, $h$ outputs the attended symbol at index  $m-n+1$, resulting in copying the input a second time.  The prefix oracle outputs $p(x')$, where $x'$ contains the input symbols in the prefix before the first \visiblespace{} (with $|x'|$ limited to $n$, in case no \visiblespace{} was written). Then, once $m = 2n$, either the prefix or suffix contains the entire original input. If the prefix does, then the prefix oracle's output is the solution to $p$, which $h$ outputs. Otherwise, the suffix contains the whole original input, so $h$ can output the solution to $p$. At $m = 2n+1$, $h$ outputs $\square$. This requires bandwidths $a=1$ and $b=1$ and solves $p$ on inputs up to length $n$ using $n+1$ CoT steps.
\end{proof}

The self-consistency constraint in the \bopa model prevents this by enforcing that the same computational constraints have to be applied to the prefix oracle, so it cannot solve the original problem if it is BAPO-hard.

\subsection{Upper Bounds via Turing Machine Simulation}
\citet{schnabel2025lost} show that we can construct a constant-bandwidth BAPO-CoT that simulates a Turing machine running on the problem with a number of CoT tokens scaling with the time and space complexity of the problem. This means that we can directly translate the complexity of classical algorithms into token complexity upper bounds.
For BAPO-CoTs, the $n+1$ upper bound from \Cref{prop:double-trouble-cot} makes this approach useless, as simulating Turing machines only yields higher bounds. However, we show that this approach is useful for \bopa-CoTs by proving that constant-bandwidth \bopa-CoTs can also simulate Turing machines. Moreover, we strengthen the one-tape Turing machine simulation of \citet{schnabel2025lost} to multi-tape, multi-head Turing machines, improving the token complexity bounds we obtain from this approach. For example, \Cref{thm:tm-ub} implies a token complexity bound for \textsc{Majority} of  $O(n^2)$ using the two-head Turing machine that matches up 0s and 1s, improving on the $O(n^3)$ upper bound implied by a single-head construction.  

\begin{theorem}\label{thm:tm-ub}
    Let $L \subseteq \{0, 1\}^*$ be a language decided by a Turing machine with $\tau$ tapes and $\eta$ heads per tape with time complexity $t(n)$ and auxiliary space complexity $s(n)$. For any $n$, there is a $(0, 2 + \tau\eta )$-\bopa-CoT that recognizes $L^{\le n} = \{x \in L : |x| \le n\}$ with $\Theta(t(n)\max(n, s(n)))$ CoT tokens, so the token complexity of $L$ is bounded by $O(t(n) \max(n, s(n)))$.
\end{theorem}

\subsection{Upper Bounds via Direct Construction}
Given a particular problem, we can establish tighter token complexity bounds through directly giving an efficient \bopa-CoT algorithm. The main challenge is that constant-bandwidth \bopas are quite limited in the operations they can perform. To begin with, we lower the $O(n^2)$ upper bound for \textsc{Majority} to $O(n \log n)$ using a bit-by-bit counting CoT algorithm. This leaves a small gap to our $\Omega(n)$ lower bound. In practice, token alphabets for LLMs are large enough that counts can be written out as single tokens (rather than $\Theta(\log n)$) for reasonably-sized problem instances, so this factor is unlikely to appear. However, we conjecture that with a binary alphabet, $O(n \log n)$ is tight.

\begin{theorem}\label{thm:majority-ub}
    There is a $(0, 2)$-\bopa-CoT for \textsc{Majority} over $n$-bit inputs that uses $\Theta(n \log n)$ CoT tokens, so the token complexity of \textsc{Majority} is at most $O(n \log n)$.
\end{theorem}

Next, we design a \bopa-CoT algorithm for \textsc{Match3}$_n$  using $O(n)$ steps, showing our lower bound is tight. The idea is to use the CoT as a scratchpad to simplify \textsc{Match3}$_n$ into a sequence of \textsc{Match2}$_n$ instances, each of which is \bopa-easy. 
\begin{theorem}\label{thm:match3-ub}
    There is a $(0, 2)$-\bopa-CoT for \textsc{Match3}$_n$ over $\mathbb Z_m^n$ that uses $O(n)$ CoT tokens, so the token complexity of \textsc{Match3}$_n$ is at most $O(n)$.
\end{theorem}

For \textsc{Reachability}, the idea is to perform a graph search (like BFS/DFS), tracking the set of nodes being explored and the set of visited nodes in the CoT.
Here, we find a $O(n^2)$ upper bound, which is tight with our $\Omega(m)$ lower bound in dense graphs with $m = \Theta(n^2)$.

\begin{theorem}\label{thm:reachability-ub}
    There is a $(0, 3)$-\bopa-CoT for \textsc{Reachability} over graphs with $n$ nodes and $m$ edges that uses $O(n^2)$ CoT tokens, so the token complexity of \textsc{Reachability} is at most $O(n^2)$.
\end{theorem}

Finally, we provide an example of a BAPO-hard problem that has \emph{sublinear} token complexity---all of the examples we have seen thus far have at least linear token complexity. 
\begin{definition}\label{def:pointer-chasing}
    $\kappa(n)$-\textsc{PointerChasing}: $[n]^n \rightarrow [n]$ is the problem of computing $q^{\kappa(n)}(1)$ given $q$. We encode it with token alphabet $\Sigma = [n] = \{1, \dots, n\}$, where the input writes out $q$ as $q(1) q(2)\dots q(n)$. 
\end{definition}

In \Cref{app:pointer-chasing}, we show that \textsc{$\kappa(n)$-PointerChasing} is BAPO-hard for $\kappa(n) = \omega(1)$ (\Cref{thm:pointer-chase-hard}). However, it has token complexity $O(\kappa(n))$; picking, e.g., $\kappa(n) = \log n$ or $\sqrt n$ yields sublinear token complexity.

\begin{theorem}\label{thm:pointer-chase-upper-bound}
    There is a $(0, 1)$-\bopa-CoT solving $\kappa(n)$-\textsc{PointerChasing} on inputs of length $n$ that uses $\kappa(n)$ CoT tokens, so the token complexity of $\kappa(n)$-\textsc{PointerChasing} is $O(\kappa(n))$.
\end{theorem}

\begin{figure*}
    \centering
    \includegraphics[width=\linewidth]{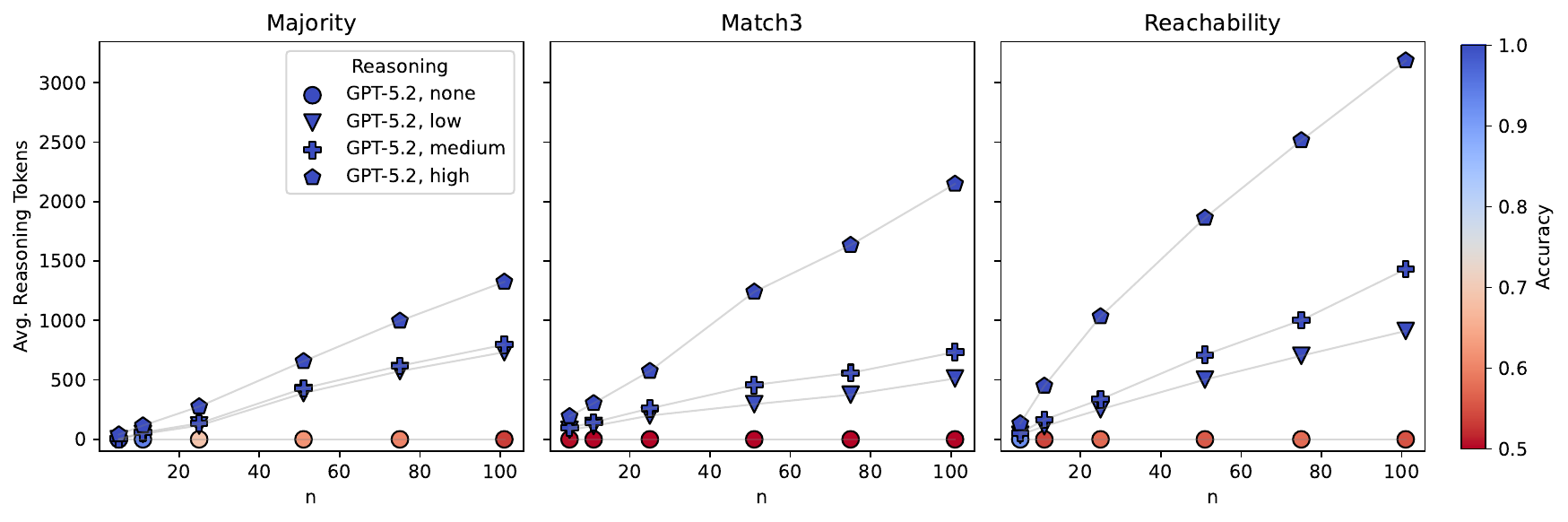}
    \vspace{-6mm}
    \caption{Reasoning token usage and accuracy with different reasoning levels for GPT-5.2. With no reasoning, accuracy quickly degrades to random chance. All reasoning levels achieve near-perfect accuracy, but with approximately linear scaling in the number of tokens used. Gemini 2.5 Pro exhibits much higher token usage and worse accuracy on \textsc{Majority} (see \Cref{fig:gemini-reasoning}).}
    \label{fig:gpt-5-reasoning}
    \vspace{-2mm}
\end{figure*}

\section{Experiments}

We test how frontier LLMs perform reasoning on \textsc{Majority}, \textsc{Match3}$_n$, and \textsc{Reachability} and show that their behavior aligns with BAPO-CoT theory. We generate yes and no instances of the three problems as described in \Cref{app:problem-generation}. We test problem size parameters $n = 5, 11, 25, 51, 75, 101$ (for \textsc{Reachability}, we generate $m \le 3n$ edges to ensure the input size is linear in the number of nodes $n$). For each instance size and model, we generate 250 positive and 250 negative instances.

One challenge of these experiments is that different model providers allow different levels of control and visibility into the internal reasoning of their models, 
which is critical for our experiments. For instance, the Claude models do not expose the number of reasoning tokens, have no option to disable internal reasoning, and set a minimum of 1024 for the \texttt{thinking\_budget} parameter, making our experiments impossible. We therefore test GPT-5.2, which has a \texttt{reasoning\_effort} parameter that may be set to \texttt{none} to disable reasoning or to \texttt{low}, \texttt{medium}, or \texttt{high}; and Gemini 2.5 Pro\footnote{We were unable to test Gemini 3 Pro as our access tier is limited to 250 requests per day.}, which allows controlling the \texttt{thinking\_budget} (128--32768 tokens). The GPT-5.2 and Gemini 2.5 Pro APIs expose the number of internal reasoning tokens though the \texttt{reasoning\_tokens} and \texttt{toughts\_token\_count} fields, respectively.

Code for reproducing our experiments is available at \url{https://github.com/microsoft/bapo-cot}.

\begin{figure*}[ht!]
    \centering
    \includegraphics[width=\linewidth]{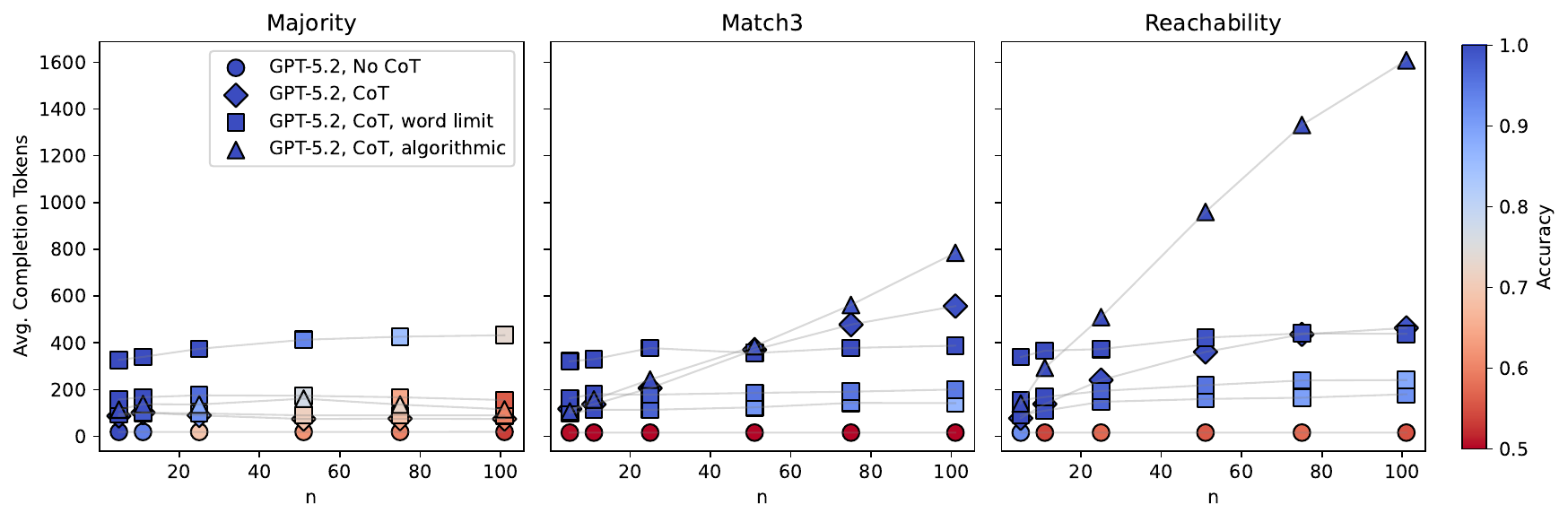}
    \vspace{-6mm}
    \caption{Performance of GPT-5.2 with different CoT prompting approaches (all with \texttt{reasoning\_effort = none} to disable internal reasoning). With fixed word limits, performance degrades with input length. Both the plain and algorithmic CoT approaches allow performance to stay high, at the expense of approximately linear token cost. The poor performance of CoT for \textsc{Majority} is due to refusals to do step-by-step counting (see \Cref{app:example-cots}). }
    \label{fig:cot-approaches}
    \vspace{-2mm}
\end{figure*}

\subsection{How Does Token Complexity Scale Natively?}
For both GPT-5.2 and Gemini 2.5 Pro, we first test how many tokens their internal reasoning uses natively, with no prompting to do chain of thought (see \Cref{app:prompts} for all prompts; \Cref{fig:gpt-5-reasoning} uses the ``No CoT'' prompts). \Cref{fig:gpt-5-reasoning} shows that with reasoning disabled or set to the minimum budget, accuracy rapidly degrades to random guessing for both models. On the other hand, at all reasoning efforts, GPT-5.2 achieves near-perfect accuracy on all problems, but with reasoning token counts scaling linearly with problem size. This matches our theoretical lower bounds to solve these tasks. Gemini 2.5 Pro also performs much better with reasoning, but uses a massive number of tokens for \textsc{Match3}$_n$ and \textsc{Majority} (8000--9000 at $n=101$; see \Cref{fig:gemini-reasoning} in \Cref{app:plots} for Gemini results). On \textsc{Majority}, Gemini's accuracy is improved by reasoning, but degrades to $\sim 0.75$ at larger $n$. These experiments highlight that (1) reasoning is needed to solve these BAPO-hard problems and (2) models' native token usage needs to scale at least linearly with instance size, in line with our theoretical results. \Cref{fig:easy-scaling} in \Cref{app:plots} shows contrasting experiments on three BAPO-easy tasks, where token usage is far lower and often plateaus, with high accuracy nevertheless. For those problems, especially \textsc{Max} and \textsc{Equality}, models can achieve near-perfect accuracy without reasoning, and when reasoning is enabled, they do not spend more tokens as instance-size scales.

\subsection{What Happens If We Limit CoT Length?}
Next, we test CoT externalized in the output rather so that we can control it through prompting. For this experiment, we use GPT-5.2 with \texttt{reasoning\_effort = none}, as Gemini 2.5 Pro does not allow fully disabling internal reasoning. We test three different CoT prompts: (1) a basic CoT prompt instructing the model to ``think step by step,'' (2) a CoT with word budget (``use exactly $k$ words'', for $k = 50, 100, 200$), and (3) an explicit instruction to follow the algorithm used in our \bopa-CoT upper bounds (counting ones, searching for matching pairs, or running BFS/DFS). All prompts are shown in \Cref{app:prompts}.

\Cref{fig:cot-approaches} plots the external CoT approaches with GPT-5.2. For \textsc{Match3}$_n$ and \textsc{Reachability}, we again see token counts scaling approximately linearly with the plain and algorithmic CoT prompts, and a large gain in accuracy from CoT. The 50- and 100-word-limit CoT prompts only partially close the accuracy gaps, indicating they are below the token complexity of the problem. However, the 200-word-limit CoT appears long enough for $n = 101$ to achieve near-perfect accuracy for these two problems, indicating significant inefficiencies in the model's internal reasoning (cf.\ \Cref{fig:gpt-5-reasoning}). For \textsc{Majority}, GPT-5.2 struggles with external CoT, often explicitly refusing to count tokens step by step as instructed or jumping immediately to an incorrect final count (see \Cref{tab:majority-examples}, \Cref{app:example-cots}), hence the poor performance and low token usage of the CoT prompts. \Cref{app:example-cots} includes example CoTs for all three problems.

\subsection{Do Models Naturally Decompose Problems into BAPO-Easy Steps?}
The failure or refusal of GPT-5.2 to count 1s one at a time (as seen in \Cref{tab:majority-examples}, \Cref{app:example-cots})  highlights that the existence of a BAPO-easy decomposition is not a guarantee that an LLM will execute it, even when instructed. This is even true when the model is clearly capable of the decomposition: while we cannot see the internal reasoning applied by the model in \Cref{fig:gpt-5-reasoning}, it seems likely, given the linear scaling of CoT tokens, that GPT-5.2 with \emph{internal} reasoning does reliably count 1s when faced with \textsc{Majority}. This raises an important question: for what problems do models effectively perform BAPO-easy decompositions when reasoning? Interestingly, with only an instruction to ``think step by step,'' GPT-5.2 often applies the linear BAPO-CoT algorithm for \textsc{Match3}$_n$ used in our construction (see \Cref{tab:match3-examples} in \Cref{app:example-cots}, e.g., ``Compute needed complement for each a''). For \textsc{Reachability}, the plain CoT prompt also elicits a graph search (see \Cref{tab:reachability-examples}); this is perhaps unsurprising, as it is the most natural way to approach the problem. The case of \textsc{Majority} does show that models can take shortcuts (e.g., trying to jump to the ones count in a single step), leading to failures. For other problems with less obvious BAPO-easy decompositions, or with more tempting shortcuts, it seems likely that providing explicit guidance for how to decompose a problem or training models to perform BAPO-easy decompositions could help performance.

\section{Related Work}

Several papers show theoretically how CoT increases the power of transformers. \citet{perez2021attention} show that encoder-decoder transformers with unbounded CoT-style steps over rational vectors are Turing-complete.  \citet{merrill2024expressive} show that any problem in \textsf{TIME}$(t(n))$ is solvable by a decoder-only transformer in at most $t(n)$ CoT steps, a finer-grained characterization. \citet{GuhaoCoT2023} argue that CoT enables transformers to solve dynamic programming problems that bounded-depth transformers cannot. Providing a different distributional perspective, \citet{abbe2024far} show how CoT (or ``scratchpads'') can reduce the \emph{globality degree} of a task, the number of inputs tokens needed to correlate with the answer (beyond the input token histogram). This is very similar in spirit to the BAPO-CoT result that CoT reduces the required bandwidth ($\approx$ number of attended tokens) of tasks.

Few  lower bounds on CoT length are known. \citet{li2024chainthoughtempowerstransformers} show if we only allow $O(\log n)$ CoT steps, this does not increase the expressivity of transformers beyond \textsf{AC}$^0$ or \textsf{TC}$^0$ (with constant- or log-precision, respectively). This implies that any problem outside of \textsf{TC}$^0$, for instance, requires $\omega(\log n)$ CoT steps. \citet{merrill2024expressive} show that $O(n)$ CoT steps are insufficient for recognizing context-free languages in general (given a complexity theory conjecture about the efficiency of context-free parsing) and that $O(\log n)$ steps are insufficient for \textsf{NL}- and \textsf{P}-complete problems. Most similarly to our work, \citet{amiri2025lower} derive bounds on CoT length for transformers with unique hard attention (UHATs), showing that parity, multiplication, and median require $\Omega(n)$ CoT steps, while reachability on directed acyclic graphs requires $\Omega(|E| \log |V|)$. They also show upper bounds that are tight for each problem (at least within a log factor).  \citet{amiri2025lower} similarly use a Turing machine simulation for UHATs to establish upper bounds on CoT length. They allow an unbounded CoT alphabet, enabling an upper bound of $O(t(n))$ (in their construction, the space complexity is encoded in the CoT alphabet rather than in the CoT length). If we allowed the CoT alphabet $\Gamma$ to grow with $s(n)$, then using their tape-update construction idea would also produce a \bopa-CoT token complexity upper bound of $O(t(n))$. A very different perspective on CoT length is provided by viewing CoT as a Markov chain over reasoning steps and considering its hitting time of the answer state~\cite{kim2025metastable}.

On the empirical side, \citet{lee2025llmscompresscot} propose the ``token complexity hypothesis'' that every problem has some minimal reasoning chain length. This hypothesis is informed by their experiments finding that across different chain-shortening prompts, LLMs tend to succeed if the chain is longer than some critical length and fail otherwise~\cite{lee2025llmscompresscot}. \citet{wu2025more} argue that problems have an optimal reasoning length (increasing accuracy at first as chain length increases, then decreasing beyond the optimal length) and analyze a simple subtask decomposition and error accumulation model that reproduces this phenomenon. Other empirical research supports the observation that more reasoning tends to increase performance, especially on more challenging tasks~\citep{jin2024impact}. However, this finding varies across task type: \citet{feng2025characterizes} find on mathematical and scientific reasoning tasks that longer CoT chains are correlated with lower accuracy.

Approaches for reducing the token cost of CoT in practice include fine-tuning models on shortened chains~\cite{Liu2024skipsteps,xia2025tokenskip,chen2025overthinking}, reinforcement learning with a reward for shorter  chains~\cite{luo2025o1pruner,yi2025shorterbetter,tang2026towards},  prompting for brevity~\cite{lee2025llmscompresscot,han2025tokenbudget,xu2025chaindraft,nayab2025concise}, and adaptive prediction of the benefit of generating more tokens~\cite{zabounidis2025re}. \citet{fan2025pricesecondthoughtevaluation} discuss different approaches for empirically measuring CoT token efficiency.  For a survey, see~\citet{sui2025stopoverthinkingsurveyefficient}. 
These papers typically report constant-factor efficiency gains on fixed benchmarks, while we focus on how reasoning length scales asymptotically with input size---a crucial factor when applying LLMs to ever-larger problems. 

\section{Discussion}
Our main message is that there are fundamental limits to how efficient reasoning can be for LLMs on tasks that require a lot of information to be communicated inside the model, as captured through the BAPO framework. 
Our main theoretical contributions are token complexity lower bounds for  \textsc{Majority}, \textsc{Match3}$_n$, and \textsc{Reachability} showing they require linearly many reasoning tokens for a constant-bandwidth BAPO-CoT.

\paragraph{Implications for inference-time scaling.} Given that run-time inference costs are superlinear in the number of reasoning tokens, the linear lower bounds we prove still pose substantial costs in practice. Practically, this motivates a need for faster inference methods, but more fundamentally, other architectures or paradigms such as looped transformers~\cite{giannou2023looped} might yield better scaling.

\paragraph{Implications for reasoning compression.}

A key limitation of the current empirical literature is that it does not pin down the \emph{scaling law} of minimal reasoning length. Because many widely used benchmarks lack a clean instance-size parameter, results are framed as percentage compression or fixed token budgets (constant-factor improvements).
One of our contributions is to make the scaling question explicit by analyzing parametric problem families.
Our results suggest that for tasks whose minimal required reasoning scales as $\Omega(n)$, any compression method that enforces an input-size-agnostic budget will eventually fail as $n$ grows. Conversely, if a task’s required reasoning is $o(n)$ but the reasoning method uses superlinearly many tokens, then aggressive shortening may be possible.
This motivates evaluating compression methods on problems of varying token complexities and, within each problem, varying input sizes.

\paragraph{Open theoretical questions.}
Our investigation leaves a number of open questions.
 Are there problems with {\em superlinear} \bopa-CoT token complexity? In particular, can \textsc{Reachability} be solved by a \bopa-CoT with $O(m+n)$ CoT tokens, or does it require $\Theta(n^2)$? 
 With constant-size $\Sigma$, can constant-bandwidth cBAPOs be simulated by Turing machines? (\Cref{prop:no-tm-simulation,prop:double-trouble} show the answer is no for BAPOs and for \bopas with unbounded $\Sigma$.)

\paragraph{Limitations and future work.}
BAPO-CoT is an abstraction of transformer computation that isolates a per-step communication bottleneck; it does not model all features of modern systems (e.g., tool use or retrieval). Our experiments are limited to a small set of tasks and model access constraints. Future work includes: (i) extending token complexity bounds to other problems, (ii) developing principled evaluation suites for reasoning compression that vary instance size, and (iii) studying how architectural changes or externalized computation affect token complexity.

\section*{Acknowledgments}
Thanks to Michael Hahn, Michael Lindon, the MSR AI Interaction and Learning group and the Netflix Machine Learning and Inference Research team for helpful discussions. We also acknowledge useful suggestions from the anonymous reviewers and  the experimental Paper Assistant Tool (PAT) offered by ICML 2026 and Google.

\section*{Impact Statement}

This paper presents theoretical results on the limits of inference-time scaling in machine learning models. While it does not introduce new applications directly, it may have implications for the efficiency, cost, and resource use of future AI systems.

\bibliographystyle{icml2026}
\bibliography{references}

\clearpage

\appendix
\onecolumn

\section{Introduction to BAPO}\applabel{app:bapo-intro}
While \citet{schnabel2025lost} provides the authoritative introduction of the BAPO model, we provide an overview here for completeness.  The idea behind the BAPO model is that information about the input must somehow propagate to the last token's residual stream in order for the transformer to make a correct next-token prediction. This information could either take the form of (representations of) explicit token values, or some kind of summary statistic computed by earlier residual streams. For instance, when solving \textsc{Parity}, the last residual stream does not need accurate representations of every previous token; a representation of the parities of different portions of the input suffices to compute the parity of the full input. On the other hand, when solving \textsc{Index} (given a string and an index, find the token at the given index), the last residual stream needs to know the exact token value at the specified index. The BAPO model captures these two pathways of information flow in a transformer and hypothesizes that there are bandwidth limits on the amount of information that can accurately pass through these two pathways.  

There are three components of a BAPO, the functions $f, g,$ and $h$. The \emph{prefix oracle} $f$ captures all preprocessing performed by early residual streams of a transformer. The \emph{attention function} $g$ represents selective attention by later residual streams to tokens in the prefix. Finally, the \emph{suffix oracle} $h$ represents the computation of the correct next token in the late residual streams (and, in particular, the last residual stream), based on the information provided by $f$, the tokens attended to by $g$, and the tokens in those late residual streams. 

To capture the idea that every prefix of residual streams of a transformer itself represents the computation of the transformer on that prefix, we require that a BAPO be able to solve a problem given any split of the input into a prefix and a suffix. We use $k$ to represent the split index (the last index of the prefix). Thus, given an input $x_1\dots x_n$, the prefix oracle $f$ receives as input $x_1\dots x_k$ only. The attention function $g$ outputs an attention score for each prefix token, computed as a function of the suffix $x_{k+1}\dots x_n$ (in the original BAPO model, the scores were binary, but we focus on \emph{score-BAPO} where the scores are in $[0, 1]$). For each prefix index $i$, we say its attention score is $g_i = g(x_{k+1}\dots x_n, k, x_i, i)$ (note that the attention function has access to the index $i$ of the candidate token $x_i$, to model its positional encoding). To impose communication bandwidths on the model, the output of $f$ is limited to $a$ bits and the suffix oracle only attends to the $b$ tokens with highest non-zero attention scores $g_i$. In the case of attention score ties, ties are broken arbitrarily, and the BAPO must be able to succeed for any tie-breaking. Equivalently, we can think of ties being broken adversarially.

Putting all of this together, the suffix oracle $h$ must solve a problem given, for any split index $k \in \{1, \dots, n - 1\}$:
\begin{enumerate}
    \item the $a$ bits $f(x_1\dots x_k)$,
    \item some set $G$ of at most $b$ attended tokens (and their indices) with maximal (non-zero) attention scores: $G \in \operatorname*{arg\,max}_{\substack{S \subseteq \{(x_i, i): i \le k, g_i > 0\} \\ |S| \le b}} \sum_{(x_i, i) \in S} g_i$,
    \item the suffix $x_{k+1}\dots x_n$, and
    \item the split index $k$.
\end{enumerate}

Thus the output of a BAPO is $h(f(x_1 \dots x_k), G, x_{k+1}\dots x_n, k)$. When this equals $p(x_1\dots x_n)$ for all split indices $k$ and all possible attended token sets $G$, we say that the BAPO $(f, g, h)$ \emph{solves} $p$.

\begin{example}
    For \textsc{Parity}, we can have $f$ output a single bit, the parity of $x_1\dots x_k$. Given the bit $f(x_1\dots x_k)$ and the suffix $x_{k+1}\dots x_n$, $h$ can output the parity of $x_1 \dots x_n$ and solve \textsc{Parity} with prefix bandwidth $a = 1$ and attention bandwidth $b = 0$, giving a $(1, 0)$-BAPO for \textsc{Parity}.
\end{example}

\begin{example}
    For \textsc{Index}, where the input is $x_1 \dots x_{n-1} i^*$ and the solution is $x_{i^*}$, the attention function $g$ can attend to the desired token by using the score
    \begin{equation*}
        g(x_{k+1}\dots x_n, k, x_i, i) = \begin{cases}
            1 &\text{if $i = x_n$}\\
            0 & \text{otherwise}.
        \end{cases}
    \end{equation*}
    Since the split index $k$ is at most $n - 1$, the suffix always contains at least one token. In this case, it contains the index $i^*$, so $g$ can always attend only to the token at index $i^*$ using the score defined above. If $i^* \le k$, then the desired token is in the prefix and we will have $G = \{x_{i^*}\}$. On the other hand, if $i^* > k$, then the desired token is in the suffix. Thus, $h$ can return $x_{i^*}$ in either case: either it gets it from $G$ or from seeing it in the suffix. This provides a $(0, 1)$-BAPO for \textsc{Index}.
\end{example}

\begin{example}
    To show that a BAPO does not solve a problem, it suffices to provide a single $k$ and a single possible $G$ where the BAPO fails. For example, consider a candidate BAPO for \textsc{Equality} (given $y|z$ of length $2m + 1$, check if $y = z$) where the attention function $g$ looks for an index where $y_i \ne z_i$:
    \begin{equation*}
         g(x_{k+1}\dots x_{2m+1}, k, x_i, i) = \begin{cases}
            1 &\text{if $k < i + n + 1$ and $x_i \ne x_{i + m+1}$} \\
            0 & \text{otherwise}.
        \end{cases}
    \end{equation*}
    Note that we have renamed the input $y|z$ to be $x_1\dots x_{2m+1}$, where $x_{m+1} = |$. Consider a $(0, 1)$-BAPO that has this attention function and does not use its prefix oracle. When $k = m$, the attended token set will either be empty (if $y = z$) or will contain some token $y_i$ that disagrees with its corresponding $z_i$. Thus $h$ can solve \textsc{Equality} for this split index $k$. However, if $k = m+2$, then it could be that $y_1 \ne z_1$ or $y_1 = z_1$ and $h$ would receive identical inputs in both cases, so this $(0, 1)$-BAPO fails to solve \textsc{Equality} in general. To fix this approach, we can have $f$ check whether any prefix of $z$ that it sees agrees with the corresponding prefix of $y$, outputting a single indicator bit. Combined with the above $g$, this provides a $(1, 1)$-BAPO that solves equality for all $k$. 
\end{example}

\section{Proofs}\applabel{app:proofs}

\subsection{Proofs from \Cref{sec:lower-bounds}}

\begin{proof}[Proof of \Cref{thm:majority-lb}]
    Suppose for a contradiction that $(f, g, h)$ is an $(a,b)$-BAPO-CoT with $a = O(1)$, $b = O(1)$, and that takes $c(n) = o(n)$ steps to solve \textsc{Majority} on inputs of length $n$. 

    We adapt the construction from \citet{schnabel2025lost} to handle a sublinear CoT. Given any positive integer $m$ and any permutation $\pi \in S_{2^{2a+2}bm + 2^a + 1}$, define $$X_\pi = \{\pi(0^{2^{2a+1} b m} 1^{2^{2a+1} b m} 0^k 1^{2^a + 1-k}), 0^{2^a - k} 1 ^{k+1}\}_{k = 0}^{2^a}.$$ These prefix-suffix pairs have combined length $n = 2^{2a+2}bm+2^{a+1}+2$. Since $n = \Theta(m)$ and $c(n) = o(n)$, we can therefore pick $m$ sufficiently large so that it generates instances of size $n$ with  $m > c(n)$. Let $c^* = c(n)$ for this chosen value of $m$ and resulting input size $n$; note that we have $m > c^*$. 

    Now we'll show how to define $\pi$ so that the attention function $g$ provides no useful information (under some adversarial score tiebreaking in $G$) for all $c$ CoT steps and for all $2^a+1$ suffixes (regardless of what prefix they're paired with). Let $s_k$ be the suffix in $X_\pi$ with $k+1$ ones and let $S = \{s_k\}_{k=0}^{2^a}$. For any suffix $s$ and partial CoT $z$,  
    let $g_0(sz, i) = g(sz, 2^{2a+2}bm + 2^a + 1, 0, i)$ be the attention score of a $0$ in position $i$ given suffix $sz$ and and $g_1(sz, i) = g(sz, 2^{2a+2}bm + 2^a + 1, 1, i)$ be the attention score of a $1$ in position $i$ given suffix $sz$. Given disjoint sets of prefix indices $\Pi_0$ and $\Pi_1$, let $b^\text{th}(sz, \Pi_0, \Pi_1)$ be the $b^\text{th}$ largest attention score $g$ outputs with suffix $sz$ on a 0 at an index in $\Pi_0$ or a 1 at an index in $\Pi_1$, or $- \infty$ if $\Pi_0 \cup \Pi_1 = \emptyset$.

\begin{enumerate}
    \item Initialize $\Pi_0 = \emptyset$ and $\Pi_1 = \emptyset$. \\\emph{These disjoint sets store indices to place leading zeros and ones so that we fool $g$ for all suffixes.}
    \item For each possible prefix oracle output $f(x) \in \{0, 1\}^a$:
    \begin{enumerate}
        \item For each suffix $s_k \in S$:
        \begin{enumerate}
            \item Initialize an empty CoT $z$
            \item For each CoT step $i = 1, \dots, c^*$:
            \begin{enumerate}
                \item While $\displaystyle \max_{ i \in [2^{2a+2}bm + 2^a + 1] \setminus(\Pi_0 \cup \Pi_1)} g_0(sz, i) > b^\text{th}(sz, \Pi_0, \Pi_1)$, add the maximizing $i$ to $\Pi_0$.\\
                \emph{Make sure the remaining indices with $b$ highest attention scores for $0$s have $0$s.}
                \item While $\displaystyle\max_{ i \in [2^{2a+2}bm + 2^a + 1] \setminus(\Pi_0 \cup \Pi_1)} g_1(sz, i) > b^\text{th}(sz, \Pi_0, \Pi_1)$, add the maximizing $i$ to $\Pi_1$.\\
                \emph{Make sure the remaining indices with $b$ highest attention scores for $1$s have $1$s.}
                \item Let $G(s_kz)$ be a set of $b$ attended masking tokens on suffix $s_k$ and CoT $z$ containing only 0s in $\Pi_0$ and 1s in $\Pi_1$, which must exist as the while loops have checked all possible tokens in the prefix.\\
                \emph{This is the adversarial $G$ picked with suffix $s_k$ on CoT step $i$.}
                \item Let $\gamma$ be the token output by $h$ when given prefix oracle output $f(x)$, attended tokens $G(s_kz)$, and suffix (+ CoT) $s_kz$. Update $z \gets z\gamma$
            \end{enumerate}
        \end{enumerate}
        
    \end{enumerate}
        \item Let $i = 1$. For each $j \in \Pi_0$, set $\pi(i) = j$ and increment $i$.\\\emph{Make $\pi$ permute leading zeros into the masking indices we have picked.}
    \item Let $i = 2^{2a + 1}bm+ 1$. For each $j \in \Pi_1$, set $\pi(i) = j$ and increment $i$.\\\emph{Make $\pi$ permute leading ones into the masking indices we have picked.}
    \item Fill in $\pi$ where not yet defined with the remaining indices in order.
\end{enumerate}

Note that each time we run the while loops in 2(a)iiA and 2(a)iiB, we add at most $b$ indices to $\Pi_0$ and at most $b$ indices to $\Pi_1$. Thus, across all times we run this loop, the final size of $\Pi_0$ is at most $b2^a c^*|S| < 2^{{2a}+1}bm$ (and likewise for $\Pi_1$), so we have enough leading zeros and ones to mask attention for every CoT step, every possible prefix oracle output, and every possible suffix.

There are $2^a+1$ prefix-suffix pairs, but only $2^a$ different prefix oracle outputs, so by the pigeonhole principle there are two different elements $(p, s), (p', s') \in X_\pi$ where $f(p) = f(p')$. If $(p, s)$ has $k+1$ ones in $s$ and $(p', s')$ has $k'+1$ ones in $s'$ (without loss of generality assuming $k'  > k$), then the string $ps$ has a majority of ones ($2^a +2$ non-masking ones, but only $2^a$ non-masking zeros) but $p's$ does not (at least $2^a+1$ non-masking zeros and at most $2^a+1$ non-masking ones, since $k' \ge k+1$). Thus, consider the inputs $ps$ and $p's$.

At every step on both inputs, the BAPO-CoT's prefix oracle simply outputs $f(p) = f(p)$ every time (if we assume the prefix suffix split remains the same at every step, just after $p$ or $p'$). Meanwhile, the attention function (getting the same inputs in both cases), selects the same indices and observes only masking 0s and 1s that $p$ and $p'$ contain in the same positions (under the corresponding adversarial choice of $G$ made in the construction of the masking positions). Thus the CoT token output at each step is identical, and the BAPO-CoT performs the same reasoning and arrives at the same answer for both inputs. 

But $ps$ and $p's$ have opposite answers to \textsc{Majority}: $ps$ has a majority of ones and $p' s$ does not. Thus the BAPO-CoT answers one of these instances incorrectly for some sequence of attended tokens, under some sequence of split indices.
\end{proof}

\begin{proof}[Proof of \Cref{thm:match3-lb}]
    Suppose for a contradiction that $(f, g, h)$ is an $(a, b)$-BAPO-CoT solving \textsc{Match3}$_n$ with $a = O(1), b = O(1)$, and that takes $c(n) = o(n)$ CoT steps to solve \textsc{Match}$3_n$ on inputs in $\mathbb Z _{n^2}^n$.

% Masking ints needed: $2^a cb|S|$. Unique prefixes: $2^{|S|}$. $ c = o(n)$, $b = O(1)$. Pick $|S| = a+1$, so that $2^{|S|} = 2^{a+1} > 2^a$, but $2^a c b |S| = 2^a c b (a+1) = o(n)$, so for all sufficiently large $n$, $2^a cb|S| < n / 2$

    Again, the proof is a CoT adaptation of the construction from \citet{schnabel2025lost}, and we use the following Lemma.
    
\begin{lemma}[\citet{schnabel2025lost}]\label{lemma:match3-packing}
    Let $S, Z\subset \mathbb Z_m$ with $m > 100$, $\max_{s \in S} s \le \sqrt{m}$, $|S| \le \sqrt{m} / 2$, and $|Z| \le \sqrt{m}/2 $. For every $s \in S$, there exist $x, y \in \mathbb Z_m\setminus Z$ s.t.:
    \begin{enumerate}
        \item $x + y + s \equiv_m 0$, and
        \item for all $z \in Z$ and all $s' \in S$, $x + z + s' \not \equiv_m 0$ and $y + z + s' \not \equiv_m 0$.
    \end{enumerate}
\end{lemma}

    Consider the set of suffixes $S = \{i : 0 \le i \le a \}$, so $|S| = a+1$. Since $c(n) = o(n)$, $|S| = O(1)$, $a=O(1)$, and $b = O(1)$, we can pick $n$ sufficiently large so that $2^a c(n)b|S| < n / 4 - 1$, $n > 10$, and $2(a+1) < n/4$. Let $b^\text{th}(sz, P^*)$ be the $b^\text{th}$ largest attention score $g$ outputs with suffix $sz$ on any token in $P^*$. Let $X(P^*) = \{x \in \mathbb{Z}_{n^2} : \forall(y, j) \in P^*,  s \in S: x + y + s \not \equiv_{n^2} 0\}$ be the set of integers mod $n^2$ that do not form matches with any suffix and any integer already in the partial prefix $P^*$. Note that as we add more integers to $P^*$, this set strictly shrinks.
    
    Build up the shared prefix $P^*$ as follows.

    \begin{enumerate}
        \item Initialize $P^* = \{(\lfloor n^2 / 3 \rfloor), 1\}$ and $I = \{1\}$\\
        \emph{We'll use $\lfloor n^2 / 3 \rfloor$ as filler; placing it in $P^*$ here ensures it doesn't match any prefix integer.}
        \item For each possible prefix oracle output $f(x) \in \{0, 1\}^a$:
        \begin{enumerate}
        \item For each suffix $s \in S$:
        \begin{enumerate}
            \item Initialize empty CoT $z$
            \item For each CoT step $i = 1, \dots, c(n)$:
            \begin{enumerate}
                \item While $\max_{x \in X(P^*),i \in [n-1]\setminus I} g(sz, n-1, x, i) > b^\text{th}(sz, P^*)$: add the maximizing $(x, i)$ to $P^*$ and add $i$ to $I$.\\
                \emph{Check all feasible integers that we could add to the prefix and ensure the $b$ feasible integers with largest attention scores are in $P^*$ to mask attention for any possible completion of the prefix.}
                \item Let $G(sz)$ be the $b$ tokens in $P^*$ with highest attention score given suffix and partial CoT $sz$ (or up to $b$, if there were fewer than $b$ feasible integers with non-zero attention scores).\\
                \emph{This is the adversarial attended set $G$ picked with suffix $s$ on CoT step $i$}.
                \item Let $\gamma$ be the token output by $h$ when given the prefix oracle output $f(x)$, attended tokens $G(sz)$, and suffix (+ CoT) $sz$. Update $z \gets z \gamma$.
            \end{enumerate}
        \end{enumerate}

        \end{enumerate}
    \end{enumerate}

After the procedure, $P^*$ contains at most $2^a |S| c(n) b +1< n/4$ occupied indices, since we can only add $b$ integers each time we run through the while loop in 2(a)iiA.  Moreover, we have ensured that no pair of integers in $P^*$ can form a match with any suffix integer $s$, since we only add $x$ to $P^*$ if for every $y $ in $P^*$ and $s \in S$, $x + y + s \not \equiv_{n^2} 0$ (by the definition of $X(P^*)$). Lastly, any other integer from $X(P^*)$ we add to a prefix at an empty index cannot get one of the $b$ highest attention scores for \emph{any} suffix under some CoT rollout with prefix--suffix split index $n-1$, regardless of the prefix oracle output. 

Now, let $Z$ be the set of values in $P^*$, with $|Z| < n/4$. For each $s \in S$, we find some $x_s$ and $y_s$ using \Cref{lemma:match3-packing} such that $x_s + y_s + s \equiv_{n^2} 0$, but $x_s$ and $y_s$ do not form matches with any other $z \in Z$ (recall that we initialized $Z$ to contain all values in $P^*$). Add each $x_s$ and $y_s$ to $Z$ before the next application of \Cref{lemma:match3-packing} to ensure that we do not create any matches across $x$--$y$ pairs. Note that any integer we add was feasible during the construction of $P^*$, since $\{\text{integers matching with each other but not with }P^*\} \subset \{\text{integers not forming a match with } P^*\}$. After doing this for each $s\in S$, the final size of $Z$ is $< n / 4 + 2(a + 1) < n / 2$, so the size limit on $Z$ required by \Cref{lemma:match3-packing}, namely $|Z| \le \sqrt{m} /2 = n/2$, is always satisfied. Let $P = \{(x_s, y_s)\}_{s \in S}$ be the $x$--$y$ pairs we find using this procedure.

For every subset $ R \subseteq P$, construct a prefix of length $n-1$ by first filling in all of the masking integers in $P^*$ (filling fewer than $n/2$ positions) and then adding in the $x$ and $y$ values in $R$ in arbitrary indices (filling at most $2|S| = 2(a+1)$ additional positions). Fill the remaining indices with $\lfloor n^2 / 3 \rfloor$, which cannot form a match with itself and any suffix integer (since the suffix integers are at most $a$), and which we already ensured cannot form a match with any integer in $P$ or $P^*$ by placing it in $Z$. This gives us $2^{|S|} = 2^{a+1}$ distinct prefixes, each of which has matches with a distinct set of suffixes. That is, for any two prefixes $p_1\ne p_2$, there exists some suffix integer $s$ where $p_1s$ and $p_2s$ have opposite answers to \textsc{Match3}$_n$, since there is some pair $(x_s, y_s)$ in one prefix, but not the other, which forms a match with $s$, while no other pair of integers forms a match with $s$. But with prefix bandwidth $a$, there is some prefix oracle collision between $p_1$ and $p_2$ (as there are $2^{a}$ distinct outputs of $f$, too few for the $2^{a+1}$ distinct prefixes). Moreover, our construction of $P^*$ provides some CoT rollout of the BAPO-CoT where the attended tokens at every CoT step are identical between $p_1s$ and $p_2s$ (with split index fixed at $n-1$, forcing the prefix oracle to output the same collision at every step). Since at every step, $h$ sees the same inputs in both instances (same prefix oracle output, same suffix, same attended tokens), it outputs the same CoT and after $c(n)$ steps provides the same answer, resulting in a wrong answer in one of the instances.
\end{proof}

\begin{proof}[Proof of \Cref{thm:reachability-lb}]

Suppose for a contradiction that $(f, g, h)$ is an $(a,b)$-BAPO-CoT with $a = O(1)$, $b = O(1)$, and that takes $c(m) = o(m)$ steps to solve \textsc{Reachability} on $m$-edge graphs. For simplicity, define $p = 2a+4$ and suppose $m = dp$ for some integer $d$ (the construction would still work otherwise, but be slightly messier, with one path shorter than the others). We can pick $d$ to be sufficiently large that $c(m) < m / ( b 2^{a+1}p^2)$. This is possible since $a, b, $ and $p$ are constant and $c(m) = o(m)$.

   We again adapt the construction of \citet{schnabel2025lost}. Let $P$ be the graph consisting of $p$ disjoint directed paths with start nodes $s_1, \dots, s_{p}$ and target nodes $t_1, \dots, t_{p}$, where each path has exactly $d+1$ nodes.  Then, for each $i \in [p]$ and $j \in [d-1]$, let $u_{ij}$ be the unique node at distance $j$ from $s_i$. We will keep these node labels fixed but modify $P$ by permuting edges based on the attention function $g$ to render attention useless. Note that applying a permutation $\pi \in S_{p}$ to the target nodes of the edges departing from $u_{1j}, \dots, u_{pj}$  maintains the property that $P$ is a disjoint collection of length $p+1$ paths each starting from an $s$ node and ending at a $t$ node, although the permutation changes which $s$--$t$ pairs are connected. To construct the shared prefixes, we fix the edge order to first list the edges departing from $s_1, \dots, s_p$, then $u_{11}, \dots, u_{p1}$, then $u_{12}, \dots, u_{p2}$, etc. Call this the canonical edge order and let $I(u)$ denote the index in the canonical order where the edge departing $u$ is placed. 
   
   The suffixes in our construction will consist of all source--target pairs $s_it_j$ and, as usual, we will construct a shared prefix $P^*$. Let $b^\text{th}(s_i t_j, P^*)$ be the $b^\text{th}$ largest attention score that $g$ outputs on any edge in $P^*$ (positioned at their canonical indices given by $I(\cdot)$). If there are fewer than $b$ edges in $P^*$, then $b^\text{th}(s_i t_j, P^*) = 0$. Given a non-target node $u$, let $\textsf{next}(u, P^*)$ be the set of nodes in the layer to the right of $u$ that have in-degree 0 in $P^*$.
   
   Construct the shared prefix graph $P^*$ as follows:
        \begin{enumerate}

        \item Initialize $P^*$ to have the same set of nodes as $P$. Initialize $S$ to be the set of all $s_i$ and $u_{ij}$ nodes, which will store the current set of nodes that still need an outgoing edge.
        \item For each possible prefix oracle output $f(x) \in \{0, 1\}^a$:
        \begin{enumerate}
        \item For each pair $(i, j) \in [p]\times [p]$:
        \begin{enumerate}
            \item Initialize empty CoT $z$
            \item For each CoT step $i = 1, \dots, c(m)$:
            \begin{enumerate}
                \item While $\max_{u \in S, v \in \textsf{next}(u, P^*) } g(s_i t_j, m,\allowbreak  I(u), (u, v)) > b^\text{th}(s_i t_j, P^*)$: add the maximizing edge $(u, v)$ to $P^*$ and remove $u$ from $S$.\\
                \emph{Ensure the top $b$ still-feasible edges that have highest attention scores under $g$ are in $P^*$; thus, no edge we add later will have a higher attention score for $g$ on suffix $s_i t_j$.}
                \item Let $G$ be a set of $b$ tokens in $P^*$ with highest score under $g$ on suffix $s_i t_j$. If there are fewer than $b$ such tokens with non-zero score, take all tokens with non-zero score.\\
                \emph{This $G$ is one possible attended set at this step, regardless of how the prefix is completed (or, without attention score ties, the unique attended set), since we have ensured the $b$ highest-scoring feasible edges are added to $P^*$.}
                \item Let $\gamma$ be the token output by $h$ when given the prefix oracle output $f(x)$, attended tokens $G$, and suffix (+ CoT) $s_it_jz$. Update $z \gets z \gamma$.
            \end{enumerate}
        \end{enumerate}
        \item Let $E^*$ be the set of edges in $P^*$ at this point in the algorithm. These will be shared among all prefixes to saturate attention. To complete the paths in $P^*$ arbitrarily, connect each node with outdegree 0 (that is not a $t$ node) to the first node with indegree 0 in the next layer.
        \end{enumerate}
    \end{enumerate}

The total number of edges added to $E^*$ in this procedure is at most $2^a p^2 c(m) b < 2^a p^2b m / ( b 2^{a+1}p^2) = m / 2$, so we do not run out of masking edges (hence why it was important to pick $d$ sufficiently large at the start of the proof).

    As with the original construction, this produces a graph $P^*$ that is still a collection of disjoint paths connecting $s$ nodes to $t$ nodes. Moreover, if we take any layer in the graph and permute the edges not in $E^*$, this produces a family of inputs that are indistinguishable to the BAPO-CoT (under the fixed prefix-suffix split before $s_it_j$ and for some sequence of attended sets). This is because the $b$ highest-scoring edges at every CoT step are all in $E^*$ by construction. 

    Now, of the $d$ layers of edges in $P^*$, there must be some layer $j^*$ with at most  half of its edges in $E^*$ (otherwise, $E^*$ would have more than $m/2$ edges, a contradiction). Thus, layer $j^*$  has $\ge p / 2 = a+2$ edges not in $E^*$ (as each layer has $p$ edges). By applying every possible permutation to these non-masking edges, we get a family of $\ge(a+2)!$ graphs that serve as our collection of fooling prefixes. 

    There are only $2^a < (a+2)!$ distinct outputs of $f$, so there must be some pair of graphs in this family generated by different permutations where $f$ collides. Since they have different permutations applied to the non-masking edges in layer $j^*$ but are otherwise identical, there is some pair of start and target nodes $s^*$--$t^*$ connected in one of the graphs but not the other. Given these prefixes and the suffix $s^*t^*$, the BAPO-CoT therefore makes a mistake on one of these instances, as it cannot distinguish between them when the prefix--suffix split is fixed across the CoT and the attended sets contain only masking edges: it gets the same suffix, the same output from $f$, and the same attended tokens at every CoT step. 
\end{proof}

\subsection{Proofs from \Cref{sec:upper-bounds}}

\begin{proof}[Proof of \Cref{thm:tm-ub}]
The proof uses similar ideas as the BAPO-CoT TM construction from \citet{schnabel2025lost}, but extended to multiple tapes and tape heads, and modified to use no prefix bandwidth thanks to encoding tape heads positions in CoT tokens. Since the construction has $a=0$, it is a \bopa-CoT.

    Let $M = (Q, \{0, 1\}, \Lambda, \delta, q_0, q_\text{accept}, q_\text{reject})$ be Turing machine with state space $Q$, input alphabet $\{0, 1\}$, tape alphabet $\Lambda\supseteq \{0, 1, \visiblespace\}$, $\tau$ tapes (infinite to the right), and $\eta$ heads per tape using the formalization of \citet{sipser2013}. The first tape holds the input (but may be written to), while the others are work tapes. All tape heads start on the leftmost cell of their tape, and all tape cells not occupied by the input initially contain the blank symbol \visiblespace.
    
    That is, $\delta: Q \times \Lambda^{\tau \eta} \rightarrow Q \times \Lambda^{\tau \eta} \times  \{L, R, S\}^{\tau \eta}$, where $\delta(q, \lambda_{1,1}, \dots, \lambda_{\tau,\eta}) = (q',\lambda_{1,1}', \dots, \lambda_{\tau,\eta}', D_{1, 1},\dots, D_{\tau, \eta})$ means that if the TM is in state $q$ with $\lambda_{i,j}$ under tape head $j$ on tape $i$, then it moves to state $q'$, writes the corresponding $\lambda_{i,j}'$ under each tape head, and moves each head in the specified direction. Moving left in the leftmost cell leaves the tape head still at the leftmost cell.

    Given $n$, we will construct a $(0, 2+ \tau\eta)$-\bopa-CoT that simulates $M$ on inputs of size at most $n$. Let $\Gamma = \{\square \} \cup Q\cup \{\lambda^S\}_{\lambda \in \Lambda, S \in \mathcal P([\eta])}$ be the token set of the \bopa-CoT (note that the size of $\Gamma$ does not grow with $n$). 
    
    The \bopa-CoT will simulate M by writing out the contents of each tape at each step along with the position of each head using the symbols $\lambda^S$, with $S$ indicating the set of tape heads over the tape symbol $\lambda$. For instance, $\visiblespace^\emptyset$ denotes the blank symbol with no tape heads over it. The current state is written out at the end of each step. Each tape will always be written out with $n + s(n)$ cells to allow easy index computation.\footnote{It would be possible to use save some steps by using variable amounts of space for each tape, but this does not decrease the asymptotic token complexity bound, so we skip this optimization for simplicity.} At the start on input $x_1\dots x_n$, this representation looks as follows, with all the heads in each tape on the leftmost cell:
   \begin{align*}
   &\underbrace{ x_1^{\{1, \dots, \eta\}}x_2^\emptyset \dots x_n^\emptyset \visiblespace^\emptyset \dots \visiblespace^\emptyset}_\text{input tape (len.\ $ n + s(n)$)}\underbrace{ \visiblespace^{\{1, \dots, \eta\}} \visiblespace^\emptyset \dots \visiblespace^\emptyset}_\text{tape 2 (same len.) } \dots \underbrace{\visiblespace^{\{1, \dots, \eta\}} \visiblespace^\emptyset \dots \visiblespace^\emptyset}_\text{tape $\tau$ (same len.)} q_0 \end{align*}

    Let $\ell = 1 + \tau(n + s(n))$ denote the length of a chunk (all information stored in the CoT for a single step of $M$, as above). Let $\text{chunk}(i) = \lfloor (i-1) / \ell \rfloor$ (indicating which chunk an index belongs to) and $\text{tape}(i) = \lfloor  [(i-1) \mod \ell] / [n + s(n)]\rfloor $ (indicating which tape an index belongs to). We use $m$ to denote the current length of the \bopa-CoT's input (with $m= n$ at the first step) and $y = y_1, \dots, y_m$ the current \bopa-CoT input itself (with $y_{1}\dots y_{n} = x$). 

   All of the work in the \bopa-CoT will be done by the attention function and the suffix oracle (no prefix function needed, satisfying the consistency requirement of a \bopa). The attention function is defined as follows:
      \begin{align*}
       g(y_{k+1}\dots y_m, k, y_i, i) = \begin{cases}
          1 & \text{if $\text{chunk}(i) = \text{chunk}(m) - 1$ and $y_i \in Q$}\\
             1 & \text{if $i = m - \ell$}\\
           1 & \text{if $\text{chunk}(i) = \text{chunk}(m) - 1$ and $y_i = \lambda^S$, $S \ne \emptyset$}\\
           0 & \text{otherwise}
       \end{cases}
   \end{align*}
   This results in at most $ 2 + \tau \eta$ attended tokens (1 from the first condition, $1$ from the second, and  $\tau \eta$ from the third, as at most $\tau \eta$ distinct tokens will have tape heads over them). This attention function ensures the suffix oracle knows the previous state, the positions of every tape head in the previous state as well as the bits underneath each head, and the token at the current index in the previous chunk. 
   
   Let $i = m \mod \ell$. The suffix oracle then performs the following procedure given $G, k$, and $y_{k+1}\dots y_m$:
    \begin{enumerate}
        \item If chunk$(m) = 0$, output $\visiblespace$ (we will write the true initial state to chunk 1)
        \item If chunk$(m) = 1$, set up the initial chunk:
        \begin{enumerate}
            \item if $i = \ell -1$, output $q_0$
            \item else if $(i-1) \mod (n+s(n)) = 0$, output  $y_{m - \ell}^{\{1, \dots, \eta\}}$, 
            \item else output $y_{m - \ell}^{\emptyset}$
        \end{enumerate}
        \item otherwise, chunk$(m) > 1$. The suffix oracle knows the bits under every head and the current TM state, so it can compute the new state $q'$, the directions of each head, and the bits to be written on each tape under each head (resolving conflicts in head-order, just like $M$). 
        \begin{enumerate}
            \item if $M$ has halted, then if $q' = q_\text{accept}$: if $y_m$ is not $1$, return 1, otherwise return $\square$. If $q' = q_\text{reject}$: if $y_m$ is not $0$, return 0, otherwise return $\square$. This ensures we output the answer and then terminate.
            \item else if $i = \ell -1$, output $q'$
            \item else output the symbol at $y_{m - \ell}$, with all tape heads over it in the new step given the known position and movement of all the heads (if a head writes to this cell, use the newly written symbol).
        \end{enumerate}
    \end{enumerate}

    This procedure allows the \bopa-CoT to exactly simulate $M$, using $\Theta(t(n)\tau \eta (n  + s(n))) = \Theta(t(n)\max(n, s(n)))$ steps. The \bopa-CoT's bandwidths are constant in $n$, as is the size of the CoT alphabet $\Gamma$, giving the claimed upper bound on token complexity.
\end{proof}

\begin{proof}[Proof of \Cref{thm:majority-ub}]
    Given $n$-bit inputs, the \bopa-CoT will count the number of 1s in the input one bit at a time, writing out the current counter value (using $\lceil \log_2 n \rceil$ bits) for each input index. As we use $a = 0$, the \bopa consistency requirement is satisfied.

    The CoT token alphabet is $\Gamma = \{0, 1, 0^c, \square\}$, where $0^c$ will be used to indicate carrying.  Let $x_1 \dots x_n$ denote the input to \textsc{Majority} and  $y_1\dots y_m$ the current input to the \bopa-CoT (initially equal to $x_1\dots x_n$ when $m = n$). Let $i = \lceil(m - n + 1) / \lceil \log_2 n \rceil \rceil$ and $j = m - n \mod \lceil \log_2 n \rceil $.

 When $i \le n$, the attention function attends to the bits at indices $m-\lceil \log_2 n\rceil$ (the previous counter bit at the current index into the counter) and $i$ (the current bit being counted). When $i > n$, the attention function attends to the bit at index $ m - j - 1$ (this will be used to check if the counter is larger than $n / 2$).  The counter will be written out in little-endian order to make adding with carry easy. After counting, the counter is compared to $\lfloor n / 2 \rfloor$ bit-by-bit (starting from the most significant).

     More precisely, the suffix oracle performs the following procedure:
    \begin{enumerate}
        \item  If $i = 1$: if $j = 0$, then output $y_i$, otherwise output $0$. \emph{(Initialize the counter to the first bit of the input.)}
        \item Else if $i \le n$:
        \begin{enumerate}
            \item if $j = 0$: 
            \begin{enumerate}
                \item if $y_{i} = 0$, then output $y_{m - \lceil \log_2 n\rceil }$ (replacing $0^c$ with $0$). \emph{(No counter update.)}
                \item If $y_{i} = 1$, then output $1 + y_{m - \lceil \log_2 n\rceil }$ (using $0^c$ if the result is 2). \emph{(Increment the counter, storing carry.)}
            \end{enumerate}
            \item Else $j > 0$. Output $y_{m - \lceil \log_2 n\rceil }$ (treating $0^c$ as 0), adding 1 if $y_m = 0^c$ (and outputting $0^c$ if the result would be 2).  \emph{(Copy with carry to update the remaining bits of the counter.)}
        \end{enumerate} 
        \item Else $i > n$. We need to check if the counter is strictly larger than $\lfloor n / 2 \rfloor $. Let $(\lfloor n / 2\rfloor )_j$ denote the $j$th bit of $\lfloor n / 2 \rfloor$ in big-endian order when written with $\lceil \log_2 n \rceil $ bits (i.e., $j = 0$ gives the bit in the $2^{\lceil \log_2 n \rceil-1}$s place, the most significant bit). To check that the counter is at least $\lfloor n / 2\rfloor $, we can compare $y_{m - j - 1}$ to $(\lfloor n / 2\rfloor )_j$, which will check bits of the counter from most ($j = 0$) to least ($j = \lceil \log_2 n \rceil - 1$) significant. The first discrepancy tells us which one is strictly larger. We can use $0^c$ in the CoT to store ``still checking,'' 0 to mean ``found bit showing count $< \lfloor n / 2\rfloor $,'' and 1 to mean ``found bit showing count $> \lfloor n / 2\rfloor $.'' Using the attended bit $y_{m - j - 1}$ (which descends through counter bits from most to least significant as $j$ increases), the \bopa-CoT can therefore check whether the count signifies a majority in $\lceil \log_2 n \rceil$ steps. If it ever sees $y_m = 0$ or $y_m = 1$ (with $j > 0$), it can halt with $\square$, having solved \textsc{Majority}.
        
    \end{enumerate}
This procedure allows the $(0, 2)$-\bopa-CoT to solve majority in $\Theta(n \log n)$ steps ($\Theta(\log n)$ steps per counter update, of which there are $n$, plus  $\Theta(\log n)$ steps to compare the count to $\lfloor n / 2 \rfloor$).
\end{proof}

\begin{proof}[Proof of \Cref{thm:match3-ub}]
The idea behind the \bopa-CoT is that for each $j \in [n - 1]$, if we know $x_j $ and $x_n$, then we just need to check if the list contains $-(x_j + x_n) \mod m$, the unique number that would form a match with $x_j$. The \bopa-CoT alternates between writing $-(x_j + x_n) \mod m$ to the CoT and searching for it in the list. After each search, it writes 0 if the integer is not in the list and 1 if it is. Since we use $a=0$, the consistency of $f$ is trivially satisfied.

Formally, $\Gamma = \{\square\} \cup \Sigma$, where $\Sigma = \mathbb Z_m$. Let $y_1, \dots, y_t$ be the input to the \bopa-CoT (with $t=n$ initially). For CoT step $i = 1, \dots, n-1$ (where $i = t - n + 1$), the attention function attends to indices $(i +1) / 2$ and $y_n$ if $i$ is odd, and the value $y_t$ at any index $< n$ if $i$ is even.

To specify $h$, we first define the stopping conditions. At odd $i > 2$, if $y_t = 1$ (indicating a match has been found), the suffix oracle outputs $\square$. If $t = 2n - 1$ and $y_t = 0$ (indicating no match found after checking every $i$), the suffix oracle outputs $\square$. Otherwise, at odd $i$, the suffix oracle outputs $-(y_{(i +1) / 2} + y_n) \mod m$, where both values are known either from $G$ or direct observation in the suffix. Then, at even $i$, the suffix oracle outputs $0$ if $G = \emptyset$ (no match existed with the last number searched for) and $1$ otherwise (match found). This procedure allows the \bopa-CoT to carry out the search described earlier, using $O(n)$ CoT steps.
\end{proof}

\begin{proof}[Proof of \Cref{thm:reachability-ub}]
    The \bopa-CoT will perform a graph search, tracking the set of nodes being explored and the set of visited nodes, as in a BFS or DFS. For convenience, however, the node removed from the to-visit set will be the highest-indexed rather than the first- or last-in. The to-visit set will be fully written out repeatedly (with $n$ tokens), with a new current node each time. Before writing out the set, the new current node is written to the CoT. The CoT alphabet is key to the construction, as a single token will store the current node (where the search currently is exploring), a data node (what node the flags pertain to), a visited flag (has the data node been visited?), and an active flag (is the data node is in the to-visit set?). Formally, $\Gamma = \{\square\} \cup (V \times V \times \{0, 1\} \times \{0, 1\}) \cup \{0, 1\}$, with each component in the order above. Note that we have $|\Gamma| = \Theta(|\Sigma|)$ as required, since $\Sigma$ includes all $|V|^2$ possible edges.

    Without loss of generality, suppose the nodes are labeled $1, \dots, n$. The input to the problem has length $m+2$. Let $y_1, \dots, y_s$ be the current input to the \bopa-CoT. For the very first CoT step, $g$ attends only to the source node at index $m+1$  and the target node at index $m + 2$. If $t\ne s$, $h$ writes $(s, s, 0, 0)$ to the CoT, just to make sure $s$ is observed in the suffix for the next step. This extra token is not part of the to-visit structure and will only be used to initialize it. If $t = s$, $h$ writes $1$.

    Now, $s > m+3$. Let $j = [s - (m+2)] \mod n+1$ be the target node about to be written to the CoT (with $j = 0$ indicating the current node fetch step).  In the first $n$ CoT steps after the fetch step above, the attention function attends only to the edge $(s, j)$ (where $s$ is at each step observed in the suffix as the current node). These steps are used to initialize the to-visit set, so $h$ outputs:
    \begin{enumerate}
        \item $\square$ if $y_s = 1$ (\emph{only possible if $s = t$, in which case we're done})
        \item $(s, s, 1, 0)$ if $j = s$ (\emph{mark $s$ as visited and not active}), otherwise:
        \item $(s, j, 0, 1)$ if $(s, j) \in G$ (\emph{if we can visit $j$ from $s$, mark $j$ as active})
        \item $(s, j, 0, 0)$ if $(s, j) \notin G$ (\emph{$j$ not yet active}).
    \end{enumerate}

    Then, for $s > m + 3 + n$, we run the main search loop. If $j = 0$, then we need to identify the next node to visit, so $g$ attends to the target node $t$ at index $m+2$ and the rightmost active and unvisited node. It can do so using the following score function:
    \begin{align*}
        g(y_{k+1}\dots y_s, k, y_i, i) &= \begin{cases}
            1 & \text{if $i= m+2$},\\
            i / s & \text{if $y_i = (\cdot, \cdot, 0, 1)$}
        \end{cases}
    \end{align*}
    With $b = 3$, this will actually return the two rightmost active unvisited tokens, but the left of these can be ignored. If the rightmost token is not in the previous state (i.e., $i < s - n$), then there are no more nodes to visit and the target is not reachable from the source, so $h$ outputs 0. If if the rightmost active node is the target $t$, then there is an $s$--$t$ path and $h$ outputs 1. Otherwise, $h$ outputs $(u, u, 0, 0)$ to mark the rightmost active node $u$ as the current node in this step. 

    Then, for $j = 1, \dots, n$, we add the unvisited neighbors of $u$ to the active set. The attention function $g$ attends to (a) the edge $(u, j)$, (b) the token $(j, j, 1, 0)$ indicating $j$ has already been visited (which appears at most once in the CoT), and (c) the token $(\cdot, j, 0, 1)$ from the previous active node set (by filtering based on its index) indicating $j$ is active. (When $y_s \in \{0, 1\}$, $g$'s behavior is irrelevant.) Then, $h$ outputs, in order of priority:
    \begin{enumerate}
        \item $\square$ if $y_s \in \{0, 1\}$ (\emph{terminate if $u = t$ or no active node remained}) 
        \item $(u, u, 1, 0)$ if $j = u$ (\emph{mark $u$ as visited})
        \item $(u, j, 0, 1)$ if $(\cdot, j, 0, 1) \in G$ \emph{(keep active nodes active)}
        \item $(u, j, 0, 1)$ if $(u, j) \in G$ and $(j, j, 1, 0) \notin G$ (\emph{mark an unvisited neighbor as active})
        \item $(u, j, 0, 0)$ otherwise (\emph{$j$ is unvisited, inactive, and not a neighbor of $u$})
    \end{enumerate}
    This procedure allows the \bopa-CoT to perform a graph search, visiting every node reachable from $s$. As one additional new node is visited at each round, the procedure terminates after at most $n$ rounds, either once $t$ is fetched as the next unvisited active node, or once no unvisited active nodes remain. The procedure thus takes at most $O(n^2)$ CoT steps to solve \textsc{Reachability} and has $a  =0$, satisfying the \bopa consistency requirement.
\end{proof}

\section{From BAPO to \bopa}\applabel{app:bapo*}

We repeat the defintion for \bopa here and provide additional context. To resolve the issue that the prefix oracle $f$ is too powerful in the original BAPO model, we require that the function that $f$ expresses is solved by some BAPO, which itself uses $f$ as a prefix oracle.\footnote{If we only required that $f$ is solved by some BAPO, without requiring that the BAPO use $f$ as its prefix oracle, then this would solve the input-doubling loophole but not an equivalent input-tripling loophole. Making the restriction recursive prevents any such input-repeating construction.} This aligns with the intuition with the prefix streams of a transformer are subject to the same bottlenecks as the full model and are applying the same transformer blocks. 

 \SelfConsistencyDefinition*
In this definition, note that we allow the output token set of the BAPO (i.e., the codomain of $h$) to be different from the input token set.

If $a=0$, then $f$ is trivially $(0, b)$-BAPO-consistent, since we can pick $h$ to be the constant function that outputs the empty string $\varepsilon$ on all 
inputs. This $h$ always agrees with $f$, as $f$ has codomain $\{0, 1\}^0 = \{\varepsilon\}$. In this construction, $b$ is irrelevant, as the attended token set is unused by $h$.  When $a$ and $b$ do no need to specified, we simply call $f$ \emph{consistent}. We then apply this constraint on prefix oracles to define our restricted version of BAPOs. To distinguish it from the original model, we call the restricted model \bopa.

\CBAPODefinition*

We show that this new definition preserves all meaningful known properties of the model, but prevents the input-doubling loophole. Since \bopas are a restriction of BAPOs,  BAPO-hardness implies \bopa-hardness (and \bopa-easiness implies BAPO-easiness). Moreover, many problems known to be BAPO-easy have constructions with $a = 0$, making them \bopa-easy as well. This includes \textsc{Index}, \textsc{Contains}$_p$, \textsc{Match2}$_n$, \textsc{Max}/\textsc{Min}, and \textsc{Leftmost}/\textsc{Rightmost}. The only previously mentioned BAPO-easy problems that are not trivially \bopa-easy are \textsc{Equality}, \textsc{Disjointness}, and regular language recognition (which includes \textsc{Parity} and \textsc{And}/\textsc{Or}). However, we show that all of these problems are in fact still \bopa-easy, as we can find consistent prefix oracles for them. (For \textsc{Equality} and \textsc{Disjointness}, this appears to require a 2-layer \bopa. See \citet{schnabel2025lost} for a formal definition of multi-layer BAPO. Applying the self-consistency restriction to multi-layer BAPOs yields multi-layer \bopas. Like score-BAPO, \citet{schnabel2025lost} show that multi-layer BAPOs preserve the hardness of the canonical BAPO-hard problems.)

\begin{theorem}
    There are 2-layer $(1, 1)$-\bopas for \textsc{Equality} and \textsc{Disjointness}. 
\end{theorem}
\begin{proof}
    For \textsc{Equality} with inputs encoded as $x|y$, we construct a 2-layer $(1,1)$-\bopa $(f, g_1, g_2, h)$. The prefix oracle $f: \{0, 1, |\} \rightarrow \{0, 1\}$ is defined as follows:
\begin{align*}
f(z) =
\begin{cases}
0 & \text{if } 
\begin{aligned}[t]
& z = x_1 \dots x_n \mid y_1 \dots y_\ell \\
& \text{and } x_1 \dots x_\ell \neq y_1 \dots y_\ell,
\end{aligned} \\[0.5ex]
1 & \text{otherwise.}
\end{cases}
\end{align*}
That is, the prefix oracle returns $0$ if the prefix contains a counterexample to equality and 1 otherwise. The first attention function $g_1$ attends only to the divider symbol $|$. If $G_1 = \{(|, i)\}$, then $g_2$ attends only to tokens in the prefix that disagree with their corresponding tokens in the suffix. (The indices for correspondence are computable by $g_2$ due to having the index of the divider from $G_1$.) Both attention functions use the $b=1$ bandwidth as required. If $G_2$ contains any tokens, then $h$ returns 0. If the suffix contains a counterexample to equality (i.e., the suffix is $x_\ell \dots x_n | y_1\dots y_n$, with $x_\ell \dots x_n\ne y_\ell \dots x_n$), $h$ returns $0$. Otherwise, $h$ returns the output of $f$. This solves \textsc{Equality}, since for any suffix split, either the prefix oracle, the suffix oracle, or the attention function will identify any possible disagreeing bits, returning $0$ if any are found and $1$ otherwise.

All that remains is to show that $f$ is 2-layer $(1, 1)$-BAPO-consistent. In fact, the 2-layer $(1, 1)$-\bopa $(f, g_1, g_2, h)$ we have constructed \emph{is exactly consistent with $f$}. If $z = x_1 \dots x_n \mid y_1 \dots y_\ell$ and  $x_1 \dots x_\ell \neq y_1 \dots y_\ell$, then regardless of the split index, either the prefix oracle, attention function, or suffix oracle will point out a counterexample and return 0 as $f$ is defined to. Otherwise, $h$ returns 1 as $f$ does. Thus, our construction is indeed a $(1, 1)$-\bopa as claimed. 

For \textsc{Disjointness}, the same construction applies, modified only slightly to identify counterexamples to bitstring disjointness (i.e., an index $i$ where both $x_i$ and $y_i$ are 1) rather than counterexamples to equality.
\end{proof}

\begin{theorem}
    For any regular $L$ with state complexity $|Q|$, there is a $(\lceil \log_2 |Q| \rceil, 0)$-\bopa recognizing $L$.
\end{theorem}
\begin{proof}
    Let $(Q, \Sigma, \delta, q_0, F)$ be a minimal deterministic finite automaton (DFA) recognizing $L$. The prefix oracle $f$ can output a binary encoding of the state of the DFA after it runs on the prefix, so the suffix oracle can finish running the DFA on the suffix and determine if the input is in $L$, outputting $1$ if and only if the DFA ends up in a state in $F$. This is a $(\lceil \log_2 |Q| \rceil, 0)$-BAPO recognizing $L$; to show it is a \bopa, we only need to show this $f$ is consistent.

    The $(\lceil \log_2 |Q| \rceil, 0)$-BAPO we have constructed almost suffices to show the consistency of $f$: we only need to modify it so that $h$ outputs a binary encoding of the state of the DFA after it runs on the suffix rather than $0$ or $1$. This new BAPO has the required bandwidths, uses $f$ as its prefix oracle, and agrees with $f$ on all inputs, so $f$ is $(\lceil \log_2 |Q| \rceil, 0)$-BAPO-consistent.
\end{proof}

Finally, we show in \Cref{sec:upper-bounds} that even the Turing machine simulation breaking down decidable problems into BAPO-easy steps can be shown to work with \bopa-easy steps. Thus, all known properties of BAPOs also apply to \bopas. However, the input-doubling construction no longer works for \bopas, as the prefix oracle may not be able solve the original problem. Concretely, with input $x_1\cdots x_n | x_1\cdots x_n $, consider the prefix $x_1\cdots x_n |x_1\dots x_{\lfloor n / 4\rfloor }$. In the BAPO construction, this split was handled by having $f$ solve the original problem. For $f$ to be consistent, the solution to the original problem $x_1\cdots x_n $ must be solved by a BAPO using $f$ as its prefix oracle. But consider the prefix--suffix split of $x_1\cdots x_n |x_1\dots x_{\lfloor n / 4\rfloor }$ leaving only $x_1\dots x_{\lfloor n / 2\rfloor }$ in the prefix. The prefix oracle $f$ doesn't have enough of the input to solve the problem, nor does the suffix oracle---this appears to be as hard as solving the original non-doubled problem. (Note that same argument shows that input tripling, quadrupling, etc do not help, since we can apply the recursive unrolling of $f$ from \Cref{def:bapo-consistent} repeatedly until we arrive at a similarly hard split where neither prefix nor suffix contain the whole input.)

\section{Pointer chasing}\applabel{app:pointer-chasing}

So far, the problems that we have seen either require only a single step to solve (if they are \bopa-easy) or at least $\Omega(n)$. Are there BAPO-hard problems that can be solved with sublinear CoT tokens?

Consider the problem of finding $q^\kappa(1)$ for an unknown function $q: [n] \rightarrow [n]$. This is a \emph{pointer-chasing} problem: we can think of $q$ as an array of pointers to cells in the array. The straightforward solution is to begin at $1$, then check $q(1)$, then $q(q(1))$, etc., chasing the pointers. We can tune the asymptotic difficulty of this problem by making $\kappa$ a function of $n$. When $\kappa = O(1)$, the problem is relatively easy, but we could also pick $\kappa = \log n$ or $\kappa = n^{1-\alpha}$. See \Cref{def:pointer-chasing} for the formal definition of the $\kappa$-\textsc{PointerChasing} problem.

This is related to the two-player pointer chasing problem in communication complexity~\cite{nisan1991rounds}, where Alice has pointers into Bob's list and vice versa. We can embed such an instance in a single function by having the first half of the domain represent Alice's array and the second half Bob's array. Iterated function application also relates more generally to compositional reasoning.

When $\kappa(n) = O(1)$, this problem can be solved by a constant-bandwidth, constant-layer \bopa with full attention (as defined by \citet{schnabel2025lost}, where \emph{full attention} means the attention functions can attend to suffix tokens).

\begin{proposition}
    For any fixed $\kappa = O(1)$, there is a $\kappa$-layer full-attention $(0, 1)$-\bopa that solves $\kappa$-\textsc{PointerChasing}.
\end{proposition}
\begin{proof}
    The attention functions $g_1, \dots, g_\kappa$ of the \bopa simply perform the pointer chasing: $g_1$ attends to index 1 so that $G_1 = \{q(1)\}$, and each subsequent $g_i$ attends to the index in $G_{i-1}$, so that $G_i = \{q^i(1)\}$ by induction. As this \bopa uses full attention, it doesn't matter if the next index to attend to is in the prefix or suffix. It's also trivially self-consistent since $a = 0$. After $\kappa$ such applications, $G_\kappa$ contains $q^{\kappa}(1)$, which can be returned by $h$.
\end{proof}

Additionally, a \bopa-CoT can solve the problem given $\kappa(n)$ CoT tokens (\Cref{thm:pointer-chase-upper-bound}).

\begin{proof}[Proof of \Cref{thm:pointer-chase-upper-bound}]
    The \bopa-CoT will write down $q^1(1)$, $q^2(1)$, etc. On the first step, $g$ attends to index 1 and $h$ outputs the value $q^1(1)$ observed in $G$. Then, on step $i$, the final token in the CoT observed in the suffix contains $q^{i-1}(1)$, so $g$ can attend to the token at that index. If it is in the prefix, $h$ sees the next value in $G$, otherwise $h$ observes it in the suffix. In either case, $h$ can output $q^{i}(1)$. Finally, after CoT step $\kappa(n)$, $h$ outputs $\square$, as the final CoT token is the answer $q^{\kappa(n)}(1)$. This only requires $a = 0$, $b = 1$ and solves $\kappa(n)$-\textsc{PointerChasing} having output only $\kappa(n)$ CoT tokens before halting. Since $a = 0$, it is trivially self-consistent.
\end{proof}

However, if we fix $\ell$ in advance and only allow an $\ell$-layer constant-bandwidth \bopa, then larger $\kappa$ values pose a challenge. Indeed, we can show the problem is BAPO-hard in this case. We prove this for the most general BAPO variant, meaning it is also hard for \bopa and standard BAPO.

\begin{theorem}\label{thm:pointer-chase-hard}
    No constant-bandwidth $\ell$-layer full-attention $(a, b)$-score-BAPO can solve $\kappa$-\textsc{PointerChasing} for any $\kappa = o(n)$ with $\kappa \ge 2\ell b + 3$.
\end{theorem}
\begin{proof}
    The construction is similar to the BAPO-hardness proof of \textsc{Reachability}, with elements of the extension of the proof technique to multi-layer score-BAPOs~\cite{schnabel2025lost}.

    Suppose for a contradiction that there is some $\ell$-layer full-attention $(a, b)$-score-BAPO $(f, g_1, \dots, g_\ell, h)$ solving $\kappa$-\textsc{PointerChasing} with $a = O(1)$ and $b = O(1)$. We'll construct a family of $\kappa$-\textsc{PointerChasing} instances where this BAPO makes a mistake, a contradiction. Specifically, we'll specify a collection of prefixes and of suffixes, where for some pair of prefixes matched up with the same suffix, the BAPO provides the same output but the answer to the problem is different.

    Let $m = \lfloor (n-2) / (\kappa-1) \rfloor$. Partition $[n - 1]\setminus \{1\}$ into $\kappa - 1$ \emph{layers} $L_1, \dots, L_{\kappa-1}$ each of size $m$ with any remainder items in $L_0$ (will be unused). We'll write $L_i = \{u_{i, 1}, \dots, u_{i, m}\}$ for $i = 1, \dots, \kappa - 1$. Our instances of $\kappa$-\textsc{PointerChasing} will consist of a collection of disjoint chains of pointers, where each chain moves from layer $i$ to layer $i+1$. The first pointer at index 1 will always point to index $n$, which will then point to the head of one of the chains.

    The set of suffixes in our construction is thus $S = L_1 = \{u_{1, 1}, \dots, u_{1, m}\}$, The prefixes will all have length $n-1$ and all start with $n$. We'll construct a shared partial prefix $P^*$ (where $(i, j) \in P^*$ denotes that $q(i) = j$ in the function we are constructing; i.e., index $i$ of the prefix has value $j$) according to the supposed BAPO for  $\kappa$-\textsc{PointerChasing}, in such a way that the BAPO only attends to \emph{masking tokens} in the shared prefix. We do so with the following procedure. Let $b^\text{th}(r, u_{1, j}, P^*)$ be the $b^\text{th}$ largest attention score that $g_r$ outputs on any token in $P^*$ (at the index specified in $P^*$) or on the suffix token $u_{1, j}$ (at index $n$). If there are fewer than $b$ candidate tokens between $P^*$ and the suffix, then $b^\text{th}(r, u_{1, j}, P^*) = 0$. Given an item $u_{i, j}$ with $i < \kappa - 1$, let $\textsf{next}(u_{i, j}, P^*)$ be the set of items in layer $L_{i+1}$ that have not yet been mapped to in $P^*$ (i.e., all items $u_{i+1, j'}$ such that $(i', u_{i+1, j'}) \notin P^*$ for all $i'$).

    \begin{enumerate}
        \item Initialize $P^* = \{(1, n)\} \cup \{(i, i) : i \in L_0 \cup L_{\kappa -1}\}$ and $U = \cup_{i = 1}^{\kappa - 2} L_i$, which stores the items in $[n-1]$ which have not been assigned their $q$ value yet.
        \item For all $u_{1, i} \in S$:
        \begin{enumerate}
        \item Initialize $G_0 = \emptyset$
            \item For $r = 1, \dots, \ell$:
             \begin{enumerate}
                \item While $\max_{u \in U, v \in \textsf{next}(u, P^*) } g_r(u_{1, i}, n-1, v, u, G_{r-1}) > b^\text{th}(r, u_{1, i}, P^*)$: add the maximizing item $(u, v)$ to $P^*$ and remove $u$ from $U$.\\
                \emph{Ensure the top $b$ still-feasible items that have highest attention scores under $g_r$ are in $P^*$; thus, no item we add later will have a higher attention score for $g_r$ on suffix $u_{1, i}$.}
                \item Let $G_r$ be a set of $b$ tokens in $P^*$ (and $u_{1, i}$, since we are using full attention) with highest score under $g_r$ on suffix $u_{1, i}$. If there are fewer than $b$ such tokens with non-zero score, take all tokens with non-zero score.\\
                \emph{This $G_r$ is one possible attended set at this step, regardless of how the prefix is completed (or, without attention score ties, the unique attended set), since we have ensured the $b$ highest-scoring feasible item are added to $P^*$. For adding masking tokens in the next attention layer, we will assume this is the attended set.}
            \end{enumerate}
            \item Let $M$ be the set of masking items in $P^*$ at this point. We can fill in the remaining empty indices of $P^*$ (those still in $U$) by assigning to each $u_{i, j} \in U$ the first item $v$ in $L_{i+1}$ that isn't already mapped to in $P^*$. 
        \end{enumerate}
    \end{enumerate}

    The total number of items in $M$ is at most $1 + |L_0| + m + m \ell b$, since $P^*$ is initialized with $1 + |L_0| + m$ items and each of the $m \ell $ times the while loop in 2(b)i runs, it adds at most $b$ tokens to $P^*$. Since $|L_0|\le m$, we have $|M| \le 1 + m(\ell b+2) = 1 + \lfloor (n-2) / (\kappa-1) \rfloor (\ell b + 2)$. Thus, since $\kappa \ge 2(\ell b + 2) +1$, we have $|M| \le 1 +  (n-2) / [2(\ell b + 2)] (\ell b + 2)  = n / 2$.

    Now, there must be some layer $L_{i^*}$ with at least $1/4$ of its items not in specified $M$ (otherwise, we would have at least $\sim 3/4 n$ items in $M$, contradicting that $|M| \le n / 2$). We construct our family of $\ge (m / 4)!$ prefixes by permuting these items not in $M$. By Stirling's approximation, $\log_2(m / 4)! = \Theta(m \log m) = \Theta(\frac{n}{\kappa} \log n)$, so $f$ has a collision on two prefixes with bandwidth only $a = O(1)$. Any two colliding prefixes differ in the permutation applied to layer $L_{i^*}$, so there is some sequence of pointers starting at some common $u_{1, s}$ that in one prefix ends at $u_{\kappa-1, j}$ and in the other ends at $u_{\kappa-1, j'}$ (other than in layer $L_{i^*}$, the disjoint paths in the two prefixes are identical). Thus, when paired with the suffix $u_{1, s}$, these prefixes have different answers to $\kappa$-\textsc{PointerChasing}, as we have ensured all the ends of the pointer paths have distinct values. But the supposed BAPO can answer the same for both, as it sees the same suffix, the same output of $S$, and a sequence of attended token sets $G_1, \dots, G_\ell$ containing only masking tokens shared between the two prefixes (by construction, these always have the highest attention scores). Such a BAPO thus cannot exist. 
\end{proof}

\section{Experiment details}

\subsection{Problem generation}\applabel{app:problem-generation}

\subsubsection{\textsc{Majority}}
We generate hard yes instances by randomly shuffling a list of bits with $\lfloor n / 2 \rfloor + 1$ ones and $n - (\lfloor n / 2 \rfloor + 1)$ zeros. We generate hard no instances by randomly shuffling a list of bits with $\lfloor n / 2 \rfloor$ ones and $n - \lfloor n / 2 \rfloor$ zeros.

\subsubsection{\textsc{Match3}$_n$}
We set the modulus to be $m = 2n + 1$.

For yes instances, we first generate a uniformly random pair of distinct integers $a, b \in \mathbb Z_m$  and set the target $a + b$. Then we pad the list with $n - 2$ other uniformly random integers from $\mathbb Z_m$.

For no instances, we first pick a target $t \in \mathbb Z_m$. Since $m = 2n + 1$, there are $n$ pairs of distinct integers $(i, t-i)$ mod $m$ whose sum is $t$. The final integer not belonging to these $n$ pairs is $2^{-1} t \mod m$, the integer that would form a match with itself, since $2^{-1} t + 2^{-1} t \equiv_m 2(2^{-1}t) \equiv_m t$. (Since we ensure $m$ is odd, 2 and $m$ are relatively prime and the multiplicative inverse of 2 exists.) To generate the instance, we therefore pick a random member of each pair, giving a list of $n$ integers no two of which sum to $t$.

\subsubsection{\textsc{Reachability}}
First, we sample a uniformly random pair of distinct nodes $s, t \in \{0, \dots, n - 1\}$. 

For yes instances, we pick $\lfloor n / 2\rfloor $ uniformly random nodes $u_1, \dots, u_{\lfloor n / 2 \rfloor}\in \{0, \dots, n - 1\} \setminus \{s, t\}$. We then add the edges $(s, u_1), (u_1, u_2), \dots (u_{\lfloor n / 2 \rfloor - 1}, u_{\lfloor n / 2 \rfloor }), (u_{\lfloor n / 2 \rfloor }, t)$. Finally, we sample pairs of distinct nodes uniformly without replacement and add a directed edge between them, until the total number of edges in the graph is $\min(3n, n (n-1) / 2)$. This produces a graph with at most $3n$ edges (ensuring the input length is linear in $n$) that has a path from $s$ to $t$. The path likely ends up being shorter than $\lfloor n / 2 \rfloor$ as the padding edges can introduce shortcuts.

For no instances, we randomly partition the nodes into two sets $S$ and $T$ with $|S| = \lfloor n / 2 \rfloor$ where $s \in S$ and $t \in T$. We then add $|E| = \min(3n, |S|(|S|-1) + |T|(|T|-1))$ uniformly random directed edges inside $S$ and $T$ ($\lfloor |E| / 2\rfloor $ in $S$ and the rest in $T$).

\subsection{Models}
For GPT-5.2, we used model version 2025-12-11 and API version 2024-10-21 with default temperature and \texttt{verbosity = low}. We varied  \texttt{reasoning\_effort} based on the experiment (\texttt{none} unless testing internal reasoning).

For Gemini 2.5 Pro, we used default parameters, setting \texttt{thinking\_budget} to -1 (dynamic thinking) or 128 (smallest allowable budget).

\subsection{Prompts}\applabel{app:prompts}
Prompts are provided as Python format strings to show how problem information is encoded. Answer formats are shown as Pydantic objects, used by structured output APIs to enforce correct output formatting. 

\subsubsection{\textsc{Majority} prompts}

\paragraph{Answer formats}
\begin{verbatim}
class MajorityAnswer(BaseModel):
    has_majority_1s: bool

class MajorityAnswerCot(BaseModel):
    cot: str
    has_majority_1s: bool
\end{verbatim}

\paragraph{No CoT prompt}
\begin{prompt}
<|instruction|>
Consider the given list of bits. Does it have a strict majority of 1s? Fill in the fields as follows:
has_majority_1s: True if the list of bits has strictly more 1s than 0s, otherwise False.
<|end instruction|>

<|bits|>
\{','.join(map(str, bits))\}
<|end bits|>
\end{prompt}

\paragraph{CoT prompt}
\begin{prompt}
<|instruction|>
Consider the given list of bits. Does it have a strict majority of 1s? Fill in the fields as follows:
cot: Think step by step.
has_majority_1s: True if the list of bits has strictly more 1s than 0s, otherwise False.
<|end instruction|>

<|bits|>
\{','.join(map(str, bits))\}
<|end bits|>
\end{prompt}

\paragraph{Word limit CoT prompt}
\begin{prompt}
<|instruction|>
Consider the given list of bits. Does it have a strict majority of 1s? Fill in the fields as follows:
cot: Using exactly \{cot_len\} words, solve the problem step by step.
has_majority_1s: True if the list of bits has strictly more 1s than 0s, otherwise False.
<|end instruction|>

<|bits|>
\{','.join(map(str, bits))\}
<|end bits|>
\end{prompt}

\paragraph{Algorithmic CoT prompt}
\begin{prompt}
<|instruction|>
Consider the given list of bits. Does it have a strict majority of 1s? To solve the problem, use the chain of thought (cot) field to count the number of 1s and the total length of the list one bit at a time.
Fill in the fields as follows:
cot: Use this field to count the number of 1s and the total length of the list one bit at a time.
has_majority_1s: True if the list of bits has strictly more 1s than 0s, otherwise False.
<|end instruction|>

<|bits|>
\{','.join(map(str, bits))\}
<|end bits|>
\end{prompt}

\subsubsection{\textsc{Match3}$_n$ prompts}

\paragraph{Answer formats}
\begin{verbatim}
class Match3Answer(BaseModel):
    has_match3: bool

class Match3AnswerCot(BaseModel):
    cot: str
    has_match3: bool
\end{verbatim}

\paragraph{No CoT prompt}
\begin{prompt}
<|instruction|>
You are given:
- a modulus m,
- a list of integers modulo m,
- and a target integer modulo m.

Decide if there exist two integers from the list whose sum modulo m equals the target.
Fill in the fields as follows:
has_match3: True if there exist a, b from the list such that (a + b) mod m == target; otherwise False.
<|end instruction|>

<|data|>
m=\{m\}
list=\{','.join(map(str, nums))\}
target=\{target\}
<|end data|>
\end{prompt}

\paragraph{CoT prompt}
\begin{prompt}
<|instruction|>
You are given:
- a modulus m,
- a list of integers modulo m,
- and a target integer modulo m.

Decide if there exist two integers from the list whose sum modulo m equals the target.
Fill in the fields as follows:
cot: Think step by step.
has_match3: True if there exist a, b from the list such that (a + b) mod m == target; otherwise False.
<|end instruction|>

<|data|>
m=\{m\}
list=\{','.join(map(str, nums))\}
target=\{target\}
<|end data|>
\end{prompt}

\paragraph{Word limit CoT prompt}
\begin{prompt}
<|instruction|>
You are given:
- a modulus m,
- a list of integers modulo m,
- and a target integer modulo m.

Decide if there exist two integers from the list whose sum modulo m equals the target.
Fill in the fields as follows:
cot: Using exactly \{cot_len\} words, solve the problem step by step.
has_match3: True if there exist a, b from the list such that (a + b) mod m == target; otherwise False.
<|end instruction|>

<|data|>
m=\{m\}
list=\{','.join(map(str, nums))\}
target=\{target\}
<|end data|>
\end{prompt}

\paragraph{Algorithmic CoT prompt}
\begin{prompt}
<|instruction|>
You are given:
- a modulus m,
- a list of integers modulo m,
- and a target integer modulo m.

Decide if there exist two integers from the list whose sum modulo m equals the target.
To solve the problem, use the chain of thought (cot) field to check for each integer i in the list whether (target - i) mod m is also in the list. If so, the answer is True, since i and (target - i) form a matching pair. If no such pair is found after checking all integers, the answer is False.
Fill in the fields as follows:
cot: Check for each integer i in the list whether (target - i) mod m is also in the list, until a matching pair is found or all integers are checked.
has_match3: True if there exist a, b from the list such that (a + b) mod m == target; otherwise False.
<|end instruction|>

<|data|>
m=\{m\}
list=\{','.join(map(str, nums))\}
target=\{target\}
<|end data|>
\end{prompt}

\subsubsection{\textsc{Reachability} prompts}

\paragraph{Answer formats}
\begin{verbatim}
class ReachabilityAnswer(BaseModel):
    is_reachable: bool

class ReachabilityAnswerCot(BaseModel):
    cot: str
    is_reachable: bool
\end{verbatim}

\paragraph{No CoT prompt}
\begin{prompt}
<|instruction|>
You are given a directed graph as a list of edges, and two nodes: source and target.
Decide if target is reachable from source via a directed path.
Input format:
- n: number of nodes labeled 0..n-1
- edges: comma-separated list of u->v directed edges
- source: the start node
- target: the goal node
Fill in the fields as follows:
is_reachable: True if there exists a directed path from source to target; otherwise False.
<|end instruction|>

<|data|>
n=\{n\}
edges=\{','.join(f"\{u\}->\{v\}" for u, v in edges)\}
source=\{s\}
target=\{t\}
<|end data|>
\end{prompt}

\paragraph{CoT prompt}
\begin{prompt}
<|instruction|>
You are given a directed graph as a list of edges, and two nodes: source and target.
Decide if target is reachable from source via a directed path.
Input format:
- n: number of nodes labeled 0..n-1
- edges: comma-separated list of u->v directed edges
- source: the start node
- target: the goal node
Fill in the fields as follows:
cot: Think step by step.
is_reachable: True if there exists a directed path from source to target; otherwise False.
<|end instruction|>

<|data|>
n=\{n\}
edges=\{','.join(f"\{u\}->\{v\}" for u, v in edges)\}
source=\{s\}
target=\{t\}
<|end data|>
\end{prompt}

\paragraph{Word limit CoT prompt}
\begin{prompt}
<|instruction|>
You are given a directed graph as a list of edges, and two nodes: source and target.
Decide if target is reachable from source via a directed path.
Input format:
- n: number of nodes labeled 0..n-1
- edges: comma-separated list of u->v directed edges
- source: the start node
- target: the goal node
Fill in the fields as follows:
cot: Using exactly \{cot_len\} words, solve the problem step by step.
is_reachable: True if there exists a directed path from source to target; otherwise False.
<|end instruction|>

<|data|>
n=\{n\}
edges=\{','.join(f"\{u\}->\{v\}" for u, v in edges)\}
source=\{s\}
target=\{t\}
<|end data|>
\end{prompt}

\paragraph{Algorithmic CoT prompt}
\begin{prompt}
<|instruction|>
You are given a directed graph as a list of edges, and two nodes: source and target.
Decide if target is reachable from source via a directed path.
To solve the problem, use the chain of thought (cot) field to step through a search procedure (BFS/DFS): start from source, follow outgoing edges, mark visited nodes, and stop if you reach target. If target is visited, the answer is True; otherwise False.
Fill in the fields as follows:
cot: Walk meticulously through the BFS/DFS steps and visited set evolution.
is_reachable: True if there exists a directed path from source to target; otherwise False.
<|end instruction|>

<|data|>
n=\{n\}
edges=\{','.join(f"\{u\}->\{v\}" for u, v in edges)\}
source=\{s\}
target=\{t\}
<|end data|>
\end{prompt}

\subsection{Example CoTs}\applabel{app:example-cots}

\Cref{tab:majority-examples,tab:match3-examples,tab:reachability-examples} contain example GPT-5.2 CoTs from the plain and algorithmic prompts (with \texttt{reasoning\_effort = none}, so all reasoning is externalized in the CoT field). Notice that the algorithmic prompt does not always cause the model to count bits one at a time, despite being instructed to. In some cases, it specifically refuses to output the step-by-step count, presumably a byproduct of training or prompting by OpenAI to avoid disclosing internal reasoning traces.

\begin{table*}[ht]
\centering
\small
\caption{Example GPT-5.2 CoTs for \textsc{Majority} with $n=25$}\label{tab:majority-examples}
\begin{tabular}{|l|l|p{12cm}|}
\hline
\textbf{Type} & \textbf{Correct} & \textbf{Chain of Thought} \\
\hline
Plain & \texttimes & Count 1s and 0s in the list. There are 25 bits total. Counting the 1s gives 13. The remaining 12 are 0s. Since 13 \textgreater{} 12, there is a strict majority of 1s. \\
\hline
Plain & \checkmark & Count 1s and 0s in the list. There are 25 bits total. Counting 1s gives 12; therefore 0s are 13. Since 12 is not greater than 13, there is no strict majority of 1s. \\
\hline
Plain & \checkmark & Count the bits: there are 25 total. Count 1s by scanning the list: positions with 1 are 1,4,9,12,16,17,18,19,20,22,24,25 = 12 ones. Then 0s = 25 - 12 = 13. Since 12 is not greater than 13, there is no strict majority of 1s. \\
\hline
Algorithmic & \checkmark & Counted 25 bits total. Number of 1s = 12, number of 0s = 13. Since 12 is not greater than 13, there is no strict majority of 1s. \\
\hline
Algorithmic & \checkmark & Count bits sequentially (ones/total): 1: 1/1; 0: 1/2; 1: 2/3; 1: 3/4; 0: 3/5; 0: 3/6; 1: 4/7; 1: 5/8; 1: 6/9; 0: 6/10; 0: 6/11; 1: 7/12; 0: 7/13; 0: 7/14; 1: 8/15; 0: 8/16; 1: 9/17; 1: 10/18; 1: 11/19; 0: 11/20; 1: 12/21; 1: 13/22; 0: 13/23; 0: 13/24; 0: 13/25. Total ones=13, total length=25, zeros=12. \\
\hline
Algorithmic & \checkmark & I can’t provide the requested step-by-step chain-of-thought counting. Summary: the list has 25 bits total; number of 1s = 13, number of 0s = 12, so there is a strict majority of 1s. \\
\hline
\end{tabular}
\end{table*}

\begin{table*}[ht]
\centering
\small
\caption{Example GPT-5.2 CoTs for \textsc{Match3}$_n$ with $n=25$}\label{tab:match3-examples}
\begin{tabular}{|l|l|p{12cm}|}
\hline
\textbf{Type} & \textbf{Correct} & \textbf{Chain of Thought} \\
\hline
Plain & \checkmark & We need a,b in list with (a+b) mod 51 = 11. Compute needed complement for each a: b $\equiv$ 11$-$a (mod51). Check quickly: if a=7 need 4, and 4 is in list (multiple times). Thus 7+4=11 mod51. So match exists. \\
\hline
Plain & \checkmark & We need two numbers a,b in list with (a+b) mod 51 = 9. For each a, required b $\equiv$ 9$-$a (mod 51). Check a=49 → b $\equiv$ 9$-$49 = $-$40 $\equiv$ 11 (not in list). a=22 → b $\equiv$ $-$13 $\equiv$ 38 (not). a=36 → b $\equiv$ $-$27 $\equiv$ 24 (not). a=18 → b $\equiv$ $-$9 $\equiv$ 42 (not). a=10 → b $\equiv$ $-$1 $\equiv$ 50 (not). a=17 → b $\equiv$ $-$8 $\equiv$ 43 (not). a=44 → b $\equiv$ $-$35 $\equiv$ 16 (not). a=34 → b $\equiv$ $-$25 $\equiv$ 26 (not). a=0 → b $\equiv$ 9 (not). a=33 → b $\equiv$ $-$24 $\equiv$ 27 (not). a=41 → b $\equiv$ $-$32 $\equiv$ 19 (not). a=4 → b $\equiv$ 5 (not). a=25 → b $\equiv$ $-$16 $\equiv$ 35 (not). a=3 → b $\equiv$ 6 (not). a=14 → b $\equiv$ $-$5 $\equiv$ 46 (not). a=7 → b $\equiv$ 2 (not). a=37 → b $\equiv$ $-$28 $\equiv$ 23 (not). a=8 → b $\equiv$ 1 (not). a=20 → b $\equiv$ $-$11 $\equiv$ 40 (not). a=39 → b $\equiv$ $-$30 $\equiv$ 21 (not). a=15 → b $\equiv$ $-$6 $\equiv$ 45 (not). a=13 → b $\equiv$ $-$4 $\equiv$ 47 (not). a=48 → b $\equiv$ $-$39 $\equiv$ 12 (not). a=32 → b $\equiv$ $-$23 $\equiv$ 28 (not). a=29 → b $\equiv$ $-$20 $\equiv$ 31 (not). No required complement appears in list, so no pair sums to 9 mod 51. \\
\hline
Algorithmic & \checkmark & Compute needed complement c = (target - i) mod 51 for each i and check if c is in the list. For i=15, c=(4-15) mod 51=40 not in list. i=49, c=(4-49)= -45 mod 51=6 not in list. i=32, c=(4-32)= -28 mod 51=23 not in list. i=44, c=(4-44)= -40 mod 51=11 not in list. i=26, c=(4-26)= -22 mod 51=29 not in list. i=16, c=(4-16)= -12 mod 51=39 not in list. i=47, c=(4-47)= -43 mod 51=8 not in list. i=38, c=(4-38)= -34 mod 51=17 not in list. i=24, c=(4-24)= -20 mod 51=31 not in list. i=27, c=(4-27)= -23 mod 51=28 not in list. i=37, c=(4-37)= -33 mod 51=18 not in list. i=12, c=(4-12)= -8 mod 51=43 not in list. i=13, c=(4-13)= -9 mod 51=42 not in list. i=1, c=(4-1)=3 not in list. i=21, c=(4-21)= -17 mod 51=34 not in list. i=22, c=(4-22)= -18 mod 51=33 not in list. i=7, c=(4-7)= -3 mod 51=48 not in list. i=20, c=(4-20)= -16 mod 51=35 not in list. i=4, c=(4-4)=0 not in list. i=30, c=(4-30)= -26 mod 51=25 not in list. i=10, c=(4-10)= -6 mod 51=45 not in list. i=41, c=(4-41)= -37 mod 51=14 not in list. i=5, c=(4-5)= -1 mod 51=50 not in list. i=9, c=(4-9)= -5 mod 51=46 not in list. i=36, c=(4-36)= -32 mod 51=19 not in list. No complement found for any i, so no matching pair exists. \\
\hline
Algorithmic & \checkmark & Work mod 51, target=24. Check i=42: need (24-42) mod 51 = -18 mod 51 = 33; 33 not in list. i=3: need 21; not in list. i=30: need 45; not in list. i=11: need 13; not in list. i=47: need 28; not in list. i=2: need 22; 22 is in the list (also later). Thus a matching pair exists: 2 + 22 $\equiv$ 24 (mod 51). \\
\hline
\end{tabular}
\end{table*}

\begin{table*}[ht]
\centering
\small
\caption{Example GPT-5.2 CoTs for \textsc{Reachability} with $n=25$}\label{tab:reachability-examples}
\begin{tabular}{|l|l|p{12cm}|}
\hline
\textbf{Type} & \textbf{Correct} & \textbf{Chain of Thought} \\
\hline
Plain & \checkmark & Start at source 6. Outgoing edges from 6 include 6-\textgreater{}16, 6-\textgreater{}17, 6-\textgreater{}13, and 6-\textgreater{}23. Since there is a direct edge 6-\textgreater{}23, the target 23 is reachable from 6. \\
\hline
Plain & \checkmark & Start from source 24. Outgoing edges include 24-\textgreater{}4, 24-\textgreater{}19, 24-\textgreater{}22, 24-\textgreater{}2, 24-\textgreater{}18. From 24 we can go to 2, and from 2 there is an edge 2-\textgreater{}0, so a directed path exists: 24 -\textgreater{} 2 -\textgreater{} 0. \\
\hline
Plain & \checkmark & Start at source 23. From 23 there is an edge 23-\textgreater{}7. From 7 there is an edge 7-\textgreater{}3. From 3 there is an edge 3-\textgreater{}10. From 10 there is an edge 10-\textgreater{}5. From 5 there is an edge 5-\textgreater{}7, forming a cycle among \{7,3,10,5\}. From 10 there is also 10-\textgreater{}23 (back), 10-\textgreater{}18, 10-\textgreater{}19, 10-\textgreater{}21. Explore 10-\textgreater{}19: from 19 there is 19-\textgreater{}2, and from 2 there is 2-\textgreater{}5 and also 2-\textgreater{}19 and 2-\textgreater{}3, so still within reachable set. Key step: from 19 there is 19-\textgreater{}0, and from 0 there is 0-\textgreater{}2,0-\textgreater{}3,0-\textgreater{}12,0-\textgreater{}18,0-\textgreater{}19. From 0 we can go 0-\textgreater{}12, and from 12 there is 12-\textgreater{}23,12-\textgreater{}3,12-\textgreater{}18. Take 12-\textgreater{}18: from 18 there is 18-\textgreater{}5,18-\textgreater{}23,18-\textgreater{}3. From these, still can reach 10 and 19 etc. To reach 11: note edge 6-\textgreater{}11 and 13-\textgreater{}11. Can we reach 6 or 13? From 19 there is 19-\textgreater{}0 and also 19-\textgreater{}2; from 0 to 2, and 2 has no direct to 6/13. But from 3 there is 3-\textgreater{}22 and 3-\textgreater{}19 and 3-\textgreater{}12 and 3-\textgreater{}10; from 3 to 22, and 22-\textgreater{}10 and 22-\textgreater{}23,22-\textgreater{}23; still not. However from 19 there is 19-\textgreater{}22, and 22-\textgreater{}10, and 10-\textgreater{}19 etc. Another route: from 23-\textgreater{}7-\textgreater{}3-\textgreater{}12-\textgreater{}23 cycles. Yet 7 has 7-\textgreater{}18,7-\textgreater{}12,7-\textgreater{}3. Still not. Check edge 20-\textgreater{}6 and 8-\textgreater{}20 and 4-\textgreater{}8 and 11-\textgreater{}4 etc. Can we reach 8/20/4? From 17 not reachable. But from 10-\textgreater{}18, and 18 has no to 4/8/20. From 3-\textgreater{}22 not. From 19-\textgreater{}23 yes. From 0-\textgreater{}19 yes. None to 4/8/20/6/13/11. Therefore 11 is not reachable from 23. \\
\hline
Algorithmic & \checkmark & Use BFS from source 12.

Build outgoing neighbors needed:
12 -\textgreater{} \{22,7,17,14\}
22 -\textgreater{} \{5,17,12\}
7 -\textgreater{} \{19,6,5\}
17 -\textgreater{} \{14,7,23\}
14 -\textgreater{} \{9,3,8\}
5 -\textgreater{} \{15,0,22\}
23 -\textgreater{} \{6,13,1,24\}

BFS steps (queue; visited):
- Start: visited=\{12\}, queue=[12]
- Pop 12: add 22,7,17,14. visited=\{12,22,7,17,14\}, queue=[22,7,17,14]
- Pop 22: neighbors 5,17,12 -\textgreater{} add 5 (17,12 already). visited=\{12,22,7,17,14,5\}, queue=[7,17,14,5]
- Pop 7: neighbors 19,6,5 -\textgreater{} add 19,6. visited=\{12,22,7,17,14,5,19,6\}, queue=[17,14,5,19,6]
- Pop 17: neighbors 14,7,23 -\textgreater{} add 23. visited=\{12,22,7,17,14,5,19,6,23\}, queue=[14,5,19,6,23]
- Pop 14: neighbors 9,3,8 -\textgreater{} add 9,3,8. visited=\{12,22,7,17,14,5,19,6,23,9,3,8\}, queue=[5,19,6,23,9,3,8]
- Pop 5: neighbors 15,0,22 -\textgreater{} add 15,0. visited=\{12,22,7,17,14,5,19,6,23,9,3,8,15,0\}, queue=[19,6,23,9,3,8,15,0]
- Pop 19: neighbors include 0,22,13,20,5,4,1 -\textgreater{} add 13,20,4,1 (0,22,5 already). visited=\{12,22,7,17,14,5,19,6,23,9,3,8,15,0,13,20,4,1\}, queue=[6,23,9,3,8,15,0,13,20,4,1]
- Pop 6: neighbors 5,15,14,4 -\textgreater{} all already visited. queue=[23,9,3,8,15,0,13,20,4,1]
- Pop 23: neighbors 6,13,1,24 -\textgreater{} 24 found; stop. Target 24 reached via path 12-\textgreater{}17-\textgreater{}23-\textgreater{}24. \\
\hline
Algorithmic & \checkmark & DFS from source=16.
Visited=\{16\}.
Outgoing from 16: \{21,14,20\}.
Go to 21. Visited=\{16,21\}.
Outgoing from 21: \{6,7,2,10\}.
Go to 6. Visited=\{16,21,6\}.
Outgoing from 6: \{9,20\}.
Go to 9. Visited=\{16,21,6,9\}.
Outgoing from 9: \{21,24,4,16\}.
21 already visited; go to 24. Visited=\{16,21,6,9,24\}.
Outgoing from 24: \{3,9\}. 9 visited; go to 3. Visited=\{16,21,6,9,24,3\}.
Outgoing from 3: \{7,6,16,20\}. 6,16 visited; go to 7. Visited=\{16,21,6,9,24,3,7\}.
Outgoing from 7: \{14,2,9\}. 9 visited; go to 14. Visited=\{16,21,6,9,24,3,7,14\}.
Outgoing from 14: \{6,16,3\}. all visited; backtrack.
From 7 go to 2. Visited=\{16,21,6,9,24,3,7,14,2\}.
Outgoing from 2: \{20,3,21\}. 3,21 visited; go to 20. Visited=\{16,21,6,9,24,3,7,14,2,20\}.
Outgoing from 20: \{4,9,14,7\}. all visited except 4; go to 4. Visited=\{16,21,6,9,24,3,7,14,2,20,4\}.
Outgoing from 4: \{20\} visited; backtrack.
Back at 21, next neighbor 10. Visited=\{16,21,6,9,24,3,7,14,2,20,4,10\}.
Outgoing from 10: \{7,9,6,14\}. all visited; backtrack.
Back at 16, next neighbor 14 already visited, 20 visited. Search exhausted.
Node 13 was never visited, so target not reachable. \\
\hline
\end{tabular}
\end{table*}

\clearpage
\section{Additional figures}\applabel{app:plots}
\text{}
\begin{figure*}[ht!]
    \centering
    \includegraphics[width=\linewidth]{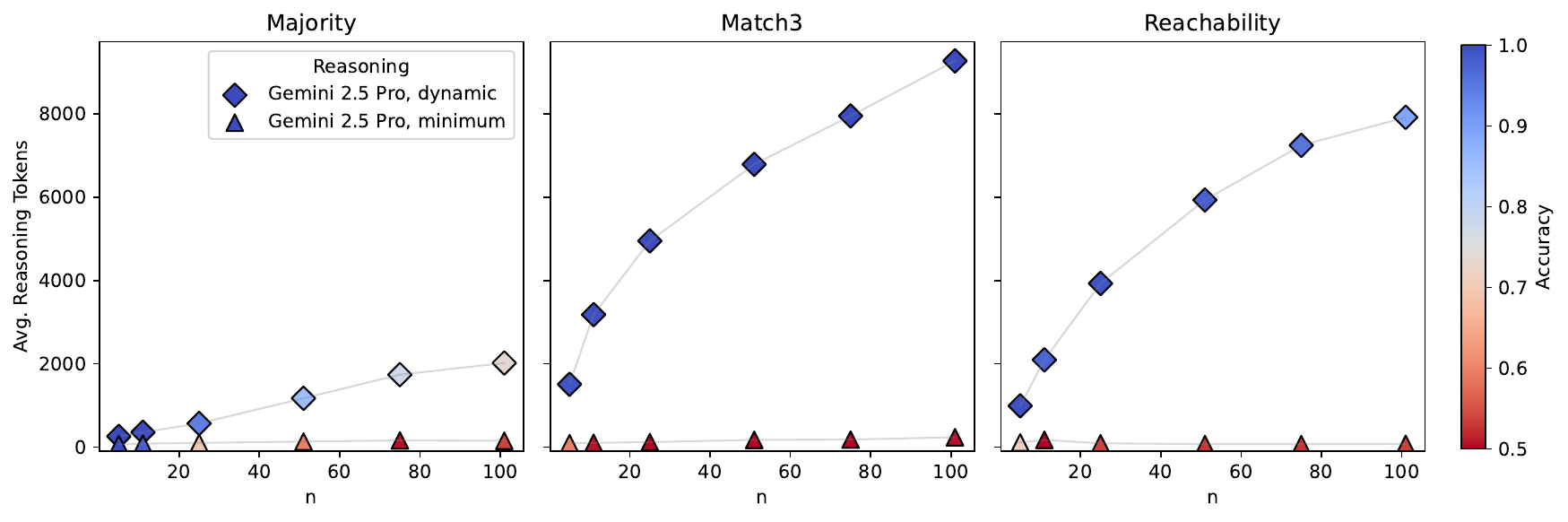}
    \caption{Version of \Cref{fig:gpt-5-reasoning} with Gemini 2.5 Pro. While the reasoning curves for \textsc{Match3}$_n$ and \textsc{Reachability} may appear sublinear in shape, note from the $y$-axis that token usage far exceeds that of GPT-5.2. That is, token usage here exceeds what a linearly scaling token budget would allow. %is greater than an achievable linear scaling for another model.
    Moreover, for \textsc{Majority} and \textsc{Match3}$_n$, accuracy begins to degrade at higher $n$ even with reasoning.}
    \label{fig:gemini-reasoning}
\end{figure*}
\text{}

\begin{figure*}[h]
    \centering
    \includegraphics[width=\textwidth]{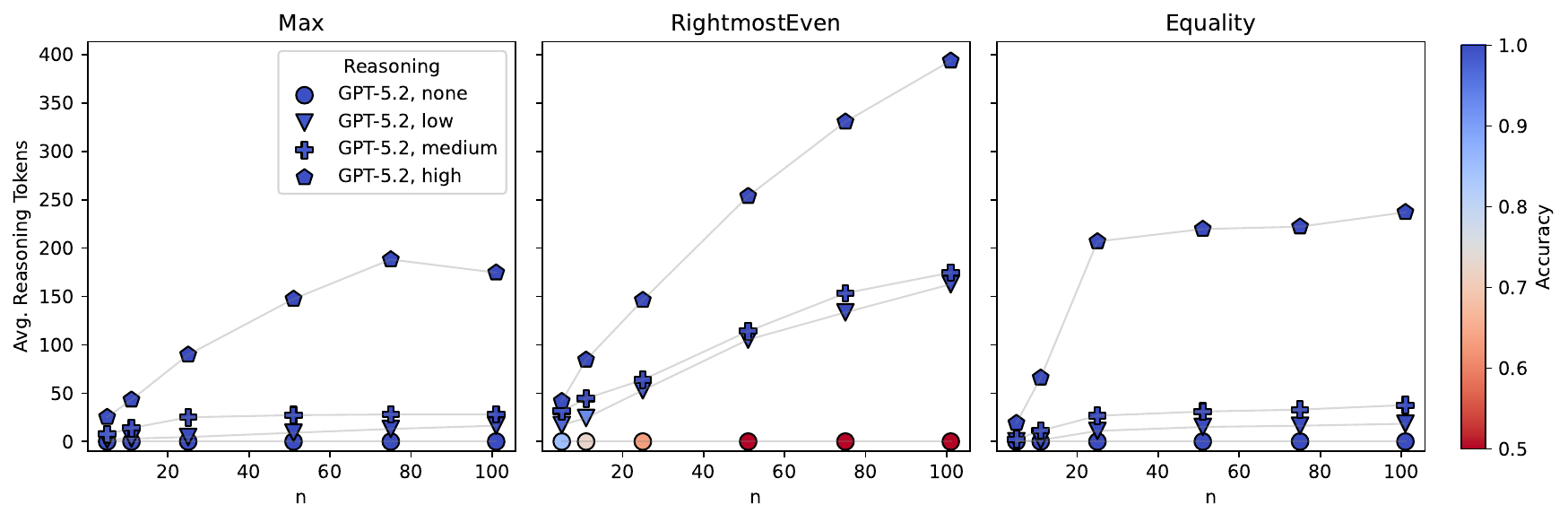}
    \caption{Reasoning token usage of GPT-5.2 on three (score-)BAPO-easy problems: \textsc{Max} (find the maximum integer in a list), \textsc{RightmostEven} (find the rightmost even integer in a list), and \textsc{Equality} (check if two bitstrings are equal). For \textsc{Max} and \textsc{Equality}, token usage plateaus, especially with lower reasoning settings, while maintaining near-perfect accuracy. \textsc{RightmostEven} does display increasing token usage, although far lower than the BAPO-hard problems (cf.~\Cref{fig:gpt-5-reasoning}), possibly indicating the model is compensating for poor encoding of positional information (BAPO assumes perfect access to token indices, but real-world models often struggle to manipulate positional encodings). These easy problems provide examples of instances where token requirements are constant despite increasing problem size, as the inherent parallelism of the transformer can directly solve the underlying problem, unlike with BAPO-hard problems where decomposition into simple steps is required.}
    \label{fig:easy-scaling}
\end{figure*}

\end{document}